\documentclass{article}

\usepackage{microtype}
\usepackage{graphicx}
\usepackage{subfigure}
\usepackage{booktabs}
\usepackage{multirow}

\usepackage{hyperref}
\usepackage{xspace} 
\usepackage{tikz}
\usetikzlibrary{bayesnet} 
\usetikzlibrary{arrows}

\usepackage{algorithm2e}

\usepackage[accepted]{icml2024}

\usepackage{amsmath}
\usepackage{amssymb}
\usepackage{mathtools}
\usepackage{amsthm}

\usepackage[capitalize,noabbrev]{cleveref}

\theoremstyle{plain}
\newtheorem{theorem}{Theorem}[section]
\newtheorem{proposition}[theorem]{Proposition}

\theoremstyle{definition}
\newtheorem{definition}[theorem]{Definition}

\theoremstyle{remark}

\usepackage[textsize=tiny]{todonotes}

\newcommand{\ours}[0]{\texttt{EQuAD}\xspace}
\newcommand{\ourst}[0]{\text{EQuAD}\xspace}

\newcommand{\oursfull}[0]{\textbf{E}ncoding-\textbf{QuA}ntifying-\textbf{D}ecorrelation\xspace}
\newcommand{\ourtitle}[0]{\text{Empowering Graph Invariance Learning with Deep Spurious Infomax}\xspace}

\usepackage{amsmath,amsfonts,bm}

\newcommand{\ind}{\perp\!\!\!\!\perp}

\def\eqref#1{equation~\ref{#1}}

\def\1{\bm{1}}

\def\vh{{\bm{h}}}

\DeclareMathAlphabet{\mathsfit}{\encodingdefault}{\sfdefault}{m}{sl}
\SetMathAlphabet{\mathsfit}{bold}{\encodingdefault}{\sfdefault}{bx}{n}

\def\gG{{\mathcal{G}}}

\DeclareMathOperator*{\argmax}{arg\,max}

\icmltitlerunning{Empowering Graph Invariance Learning with Deep Spurious Infomax}

\begin{document}

\twocolumn[
\icmltitle{\ourtitle}




\icmlsetsymbol{equal}{*}

\begin{icmlauthorlist}
\icmlauthor{Tianjun Yao}{equal,mbzu}
\icmlauthor{Yongqiang Chen}{equal,mbzu,cuhk}
\icmlauthor{Zhenhao Chen}{mbzu}
\icmlauthor{Kai Hu}{cmu}
\icmlauthor{Zhiqiang Shen}{mbzu}
\icmlauthor{Kun Zhang}{mbzu,cmu}
\end{icmlauthorlist}

\icmlaffiliation{mbzu}{Mohamed bin Zayed University of Artificial Intelligence, Abu Dhabi, UAE}
\icmlaffiliation{cuhk}{The Chinese University of Hong Kong, Hong Kong, China}
\icmlaffiliation{cmu}{Carnegie Mellon University, Pittsburgh, USA}

\icmlcorrespondingauthor{Tianjun Yao}{tianjun.yao@mbzuai.ac.ae}
\icmlcorrespondingauthor{Zhiqiang Shen}{zhiqiang.shen@mbzuai.ac.ae}
\icmlcorrespondingauthor{Kun Zhang}{kun.zhang@mbzuai.ac.ae}

\icmlkeywords{invariance learning, out-of-distribution generalization, graph representation learning}

\vskip 0.3in
]



\printAffiliationsAndNotice{\icmlEqualContribution} 

\begin{abstract}
Recently, there has been a surge of interest in developing graph neural networks that utilize the invariance principle on graphs to generalize the out-of-distribution (OOD) data. 
Due to the limited knowledge about OOD data, existing approaches often pose assumptions about the correlation strengths of the underlying spurious features and the target labels. However, this prior is often unavailable and will change arbitrarily in the real-world scenarios, which may lead to severe failures of the existing graph invariance learning methods.
To bridge this gap, we introduce a novel graph invariance learning paradigm, which induces a robust and general inductive bias. 
The paradigm is built upon the observation that the infomax principle encourages learning spurious features regardless of spurious correlation strengths. We further propose the \ours framework that realizes this learning paradigm and employs tailored learning objectives that provably elicit invariant features by disentangling them from the spurious features learned through infomax.
Notably, \ours shows stable and enhanced performance across different degrees of bias in synthetic datasets and challenging real-world datasets up to $31.76\%$. Our code is available at \url{https://github.com/tianyao-aka/EQuAD}.

\vspace{-0.1in}
\end{abstract}

\section{Introduction}
\label{intro}

Despite the enormous success of Graph Neural Networks (GNNs)~\citep{kipf2016semi,xu2018powerful,veličković2018graph}, they generally assume that the testing and training graphs are independently sampled from an identical distribution, i.e., the I.I.D. assumption, which can not be guaranteed in many real-world applications \citep{hu2021open,koh2021wilds,huang2021therapeutics}. 
Therefore, it has drawn great attention from the community to overcome the Out-of-Distribution (OOD) generalization challenge that enables GNNs to generalize to new environments outside the training distributions. Recent studies incorporate the invariance principle from causality~\cite{peters2015causal} into GNNs~\citep{wu2022handling,chen2022learning}, in which the rationale is to learn the invariant graph features or identify invariant subgraphs that only focus on the direct causes of the target label and discards the other features whose correlations with the target labels may change across different environments (e.g., graph sizes). 
Due to the non-Euclidean and abstraction nature of graph data, the environment labels for distinguishing distribution shifts are usually expensive or unavailable, as collecting these labels typically requires expert knowledge~\citep{wu2022handling,chen2022learning}. Therefore, existing approaches generally rely on intermediate partitions of invariant and spurious graph features either from input space~\cite{fan2022debiasing,chen2022learning,wu2022discovering} or latent space~\cite{li2022learning,Liu_2022,zhuang2023learning} and adopts additional assumptions in order to learn the desired graph invariant features.

A significant challenge arises in the application of these graph invariance learning algorithms, as the assumptions typically pose a strong prior about the joint distribution of spurious features $S$ and class label $Y$, i.e., $\mathbb{P}(S, Y)$, in an implicit or explicit way. In real-world scenarios, however, $\mathbb{P}(S, Y)$ can vary arbitrarily, leading to varying correlation degrees between $S$ and $Y$. Consequently, this variability may conflict with the assumptions underlying these algorithms, resulting in potential failures.  
For example, DisC~\cite{fan2022debiasing} presumes a strong correlation between $S$ and $Y$ and uses the presumption to identify a biased graph. Then DisC contrasts against the separated biased subgraph to learn the unbiased graph features.
On the other hand, environment inference~\cite{yang2022learning,li2022learning} and augmentation~\cite{wu2022discovering,Liu_2022} algorithms typically presume a weaker correlation between $S$ and $Y$ to accurately infer or augment the environments. However, the premise can be easily broken when spurious correlation strengths shift, and lead to the failure of graph invariance learning~\citep{chen2023does}. The brittleness of relying on presumed correlations as inductive bias, raises a challenging research question:

\begin{quotation}
    \textit{Is there a reliable inductive bias that remains robust against varying degrees of correlation between $S$ and $Y$, and enhances graph invariance learning?}
\end{quotation}

\textbf{Present work.} To address these challenges, we turn to self-supervised learning (SSL), which alleviates the dependence on $Y$, and eliminates the need for any assumption regarding $\mathbb{P}(S,Y)$. Notably, we show that employing global-local mutual information (MI) maximization, or the \textit{infomax principle}~\cite{hjelm2019learning,veličković2018deep,36,Bell1995AnIA} as the self-supervision objective, enables the model $f(\cdot)$ to capture spurious features with provable guarantees. Building on this insight, we have developed a new paradigm that decouples the learning of invariant features $C$ and spurious features $S$. Specifically, we first learn representations containing $S$ predominantly, then use these representations to uncover $C$.
Furthermore, we propose a flexible framework \oursfull (\ours) in order to realize the novel learning paradigm, where off-the-shelf algorithms can serve as plug-ins for specific implementations for each step. \ours consists of the following three key steps:
1) \textit{Encoding.} We first utilize infomax-based SSL to obtain graph representations fully encoding the spurious features $S$. 
2) \textit{Quantifying.} We then quantify the data samples into a low-dimensional latent space, which accurately captures the correlation degree between $S$ and $Y$ for each sample. 
3) \textit{Decorrealtion.} Finally, we retrain a GNN model from scratch to learn graph invariant representation by leveraging the spurious features obtained from previous steps. Our contributions can be summarized as follows: 

\begin{itemize}

\item We reveal that self-supervised learning, when grounded in the infomax principle, can reliably isolate spurious features under certain mild conditions. In light of the finding, we introduce a new learning paradigm for graph invariance learning, which induces a robust inductive bias that relieves the reliance on presuming spurious correlation strengths between $S$ and $Y$. (Sec.~\ref{sec:self_sup_infomax})

\item We propose a flexible learning framework called \ours to realize the learning paradigm, as well as a new learning objective that is tailored for \ours, which provably learns the invariant graph representations. (Sec.~\ref{sec:method})

\item We conduct extensive experiments on 7 synthetic datasets and 8 real-world datasets with various types of distribution shifts. The results demonstrate the superiority of our method compared to state-of-the-art approaches. Notably, our method exhibits stable and enhanced performance across different degrees of bias in the synthetic datasets, and outperforms the baseline methods by an average of $31.76\%$. (Sec.~\ref{sec:exp}). 
\end{itemize}

\section{Preliminaries}
\subsection{Notations and Problem Definition}

\textbf{Notations.} Throughout this work, we use $C$ and $S$ to denote contents (invariant factors) and styles (spurious factors) respectively, which are interchangeably with $\vh_c$ (invariant representations) and $\vh_s$ (spurious representations) in this work. $\widehat{C} $ and $\widehat{S}$ denote the estimated invariant and spurious factors, similarly for $\widehat{\vh}_c$ and $\widehat{\vh}_s$. We use $[K]:=\big\{1,2,\cdots,K\big\}$ to denote a index set, $f(\cdot)$ to denote a function, $w$ to denote a scalar value, $\mathbf{w}$ and $\mathbf{W}$ to denote a vector and matrix, respectively. A more complete set of notations are presented in Appendix \ref{app:notations}.

\textbf{Problem Definition.} We focus on OOD generalization in graph classification. Given a set of graph datasets $\mathcal{G}=\{\mathcal{G}^{e}\}_{e \in \mathcal{E}_{\mathrm{tr}} \subseteq \mathcal{E}_{\mathrm{all}}}$, a GNN model $f$, denoted as \(\rho \circ h\), comprises an encoder \(h\!\!: \mathbb{G} \rightarrow \mathbb{R}^F\) that learns a representation \(\vh_G\) for each graph \(G\), followed by a downstream classifier \(\rho\!: \mathbb{R}^F \rightarrow \mathbb{Y}\) to predict the label \(\widehat{Y}_G=\rho(\vh_G)\). The objective of OOD generalization on graphs is to learn an optimal GNN model $f^*(\cdot): \mathbb{G} \rightarrow \mathbb{Y}$ with data from training environments \(\mathcal{G}_{\text{tr}}=\{\mathcal{G}^e\}_{e \in \mathcal{E}_{\text{tr}}}\) that effectively generalizes across all (unseen) environments:
\begin{equation}\label{eq:inv_learn}
f^*(\cdot)=\arg \min _f \sup _{e \in \mathcal{E}_{all}} \mathcal{R}(f \mid e),
\end{equation}
where $\mathcal{R}(f \mid e)=\mathbb{E}_{\mathrm{G}, \mathrm{Y}}^e[ l(f(\mathrm{G}), \mathrm{Y})]$ is the risk of the predictor $f$ on the environment $e$, and $l(\cdot, \cdot)$ : $\mathbb{Y} \times \mathbb{Y} \rightarrow \mathbb{R}_{+}$ denotes a loss function.

\begin{figure}[t]
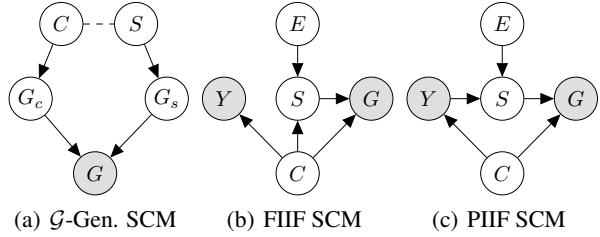

    \centering
    \subfigure[$\gG$-Gen. SCM]{\label{fig:graph_gen}
        \resizebox{!}{0.15\textwidth}{\tikz{
                \node[latent] (S) {$S$};%
                \node[latent,left=of S,xshift=0.5cm] (C) {$C$};%
                \node[latent,below=of C,xshift=-0.5cm,yshift=0.5cm] (GC) {$G_c$}; %
                \node[latent,below=of S,xshift=0.5cm,yshift=0.5cm] (GS) {$G_s$}; %
                \node[obs,below=of GC,xshift=1.05cm,yshift=0.5cm] (G) {$G$}; %
                \edge[dashed,-] {C} {S}
                \edge {C} {GC}
                \edge {S} {GS}
                \edge {GC,GS} {G}
            }}}
    \subfigure[FIIF SCM]{\label{fig:scm_fiif}
        \resizebox{!}{0.15\textwidth}{\tikz{
                \node[latent] (E) {$E$};%
                \node[latent,below=of E,yshift=0.5cm] (S) {$S$}; %
                \node[obs,below=of E,xshift=-1.2cm,yshift=0.5cm] (Y) {$Y$}; %
                \node[obs,below=of E,xshift=1.2cm,yshift=0.5cm] (G) {$G$}; %
                \node[latent,below=of Y,xshift=1.2cm,yshift=0.5cm] (C) {$C$}; %
                \edge {E} {S}
                \edge {C} {Y,G}
                \edge {S} {G}
                \edge {C} {S}
            }}}
    \subfigure[PIIF SCM]{\label{fig:scm_piif}
        \resizebox{!}{0.15\textwidth}{\tikz{
                \node[latent] (E) {$E$};%
                \node[latent,below=of E,yshift=0.5cm] (S) {$S$}; %
                \node[obs,below=of E,xshift=-1.2cm,yshift=0.5cm] (Y) {$Y$}; %
                \node[obs,below=of E,xshift=1.2cm,yshift=0.5cm] (G) {$G$}; %
                \node[latent,below=of Y,xshift=1.2cm,yshift=0.5cm] (C) {$C$}; %
                \edge {E} {S}
                \edge {C} {Y,G}
                \edge {S} {G}
                \edge {Y} {S}
            }}}
    \vspace{-0.1in}
    \caption{Structural causal models for graph generation.}
    \label{fig:scm}
\end{figure}

\subsection{Data Generating Process}

We consider the graph generation process most widely discussed in the literature~\citep{wu2022handling,chen2022learning,miao2022interpretable,fan2022debiasing,li2022learning,chen2023does}. 
As shown in Fig.~\ref{fig:scm}, the observed graph $G$ consists of an underlying invariant subgraph $G_c$ and spurious subgraph $G_s$, which are generated under the control of the invariant latent factor $C$ and spurious latent factor $S$, respectively. $C$ causally determines $Y$ while $S$ could be affected by the changes in the environment $E$.
$C$, $S$ and $Y$ can exhibit two kinds of relations, i.e., Fully Informative Invariant Features (FIIF) when $Y\ind S|C$ and Partial Informative Invariant Features (PIIF) when $Y\not\ind S|C$. More details are included in Appendix~\ref{app:data_gen}.

\section{Related Work}
\label{sec:related_work}

\noindent{\bf Graph Invariance Learning.} In recent years, there has been an increasing focus on learning graph representations that are robust to distribution shifts, especially from the perspective of invariant learning. Some works involve environment inference~\cite{yang2022learning,li2022learning} or environment augmentation~\cite{wu2022discovering,Liu_2022} algorithms, which infer environmental labels, or perform environment augmentation, and then use this information to learn graph invariant features. Another line of work adopts alternative strategies to achieve invariant learning, without directly dealing with the unobserved environments~\cite{fan2022debiasing,chen2022learning,chen2023does}. For instance, CIGA~\cite{chen2022learning} utilizes contrastive learning within the same class labels, assume samples with the same label share invariant substructures; DisC~\cite{fan2022debiasing}, conversely, leverages biased information to initially learn a biased graph for subsequent invariance learning. However, most of these methods often rely on strong assumptions about the joint distribution $\mathbb{P}(S,Y)$, which can lead to potential failures in real-world scenarios for OOD generalization. In this work, we propose a new learning paradigm which induces a robust inductive bias by eliminating the reliance on the correlation between $S$ and $Y$. 

\noindent{\bf Identifiability in Self-Supervised Learning.} Self-supervised learning with augmentations has gained huge success in learning useful graph representations~\citep{dgi,graphcl}.
Existing analysis of self-supervised learning focuses on showing the desired property such as identifying the content from style~\citep{ssl_sep}, or invariant subgraph from spurious one~\citep{chen2022learning,chen2023does,mole_identify}. In contrast, we show that infomax principle tends to learn the spurious features, which can be leveraged to learn graph invariant features.

\section{Learning Spurious Features with Self-supervision}
\label{sec:self_sup_infomax}

In this section, we delve into how self-supervision can effectively identify spurious features with provable guarantee, which serves as \textit{a key step in our proposed algorithm}. Concretely, We show that by employing a self-supervised approach based on the infomax principle, we can decouple the supervised learning and the identification of $S$, while in the meantime reducing the reliance on the presuming spurious correlation strengths between $S$ and $Y$. First, we outline the infomax principle in Eqn. \ref{infomax_eq}.
\begin{equation}
\label{infomax_eq}
\max_\theta \frac{1}{|G||\mathcal{G}|} \sum_{G \in \mathcal{G}} \sum_{i \in |G|} I\left(\widehat{\mathbf{h}}_i; \widehat{\mathbf{h}}_G\right),
\end{equation}

where $\widehat{\mathbf{h}}_i$ and $\widehat{\mathbf{h}}_g$ denote the node and graph representations respectively, and $\theta$ denotes the parameters of the encoder.
The goal of Eqn. \ref{infomax_eq} is to maximize the MI between a global representation (e.g., a graph) and local parts of the inputs (e.g., nodes), which encourages the encoder to carry information presented in all locations. Intuitively, this maximization encourages the encoder to capture information presented across all locations. However, it is important to note that the global representations learned through Eqn. \ref{infomax_eq} might favor spurious correlations rather than causally-related high-level semantics, especially if the object of interest occupies a relatively small size within the global context. More formally, we present the following theorem:

\begin{theorem}\label{thm:spu_infomax}
    Given the same data generation process as in Fig.~\ref{fig:scm} with Shannon entropy $H(S)=H(C)=\delta_f$, assuming the node representations encode proper information of the underlying latent factors, i.e., $\delta_r\geq I(\widehat{\vh}_{i};C)- I(\widehat{\vh}_{i};S)\geq \delta_l, \forall i\in G_c$ and $\delta_r\geq I(\widehat{\vh}_{i};S)- I(\widehat{\vh}_{i};C)\geq \delta_l, \forall i\in G_s$, the graph representation $\widehat{\vh}_G \in \mathbb{R}^k$ have sufficient capacity to encode $k$ independent features $\{\widehat{\vh}_G[j]\}_{j=1}^k$ with $H(\widehat{\vh}_G[j])\leq \delta_f$, then, if $|G_s|/|G_c|> \delta_r/\delta_l$, the graph representation elicited by the infomax principle~(Eqn.~\ref{infomax_eq}) exclusively contain spurious features $S$, i.e.,
    \vspace{-0.1in}
        $$
        \mathbf{h}_s=\argmax _\theta \frac{1}{|G||\mathcal{G}|} \sum_{G \in \mathcal{G}} \sum_{i \in|G|} I\left(\widehat{\mathbf{h}}_i; \widehat{\mathbf{h}}_G\right).
        $$
    
\end{theorem}\vspace{-0.1in}

The proof of Theorem~\ref{thm:spu_infomax} is given in Appendix~\ref{proof:spu_infomax}. The characterization of features in the graph representation is motivated by the feature learning literature that neural networks tend to repeatedly encode features~\citep{addepalli2023feature_replication,rfc2}.
The key observation from Eqn.~\ref{infomax_eq} is based on the inductive bias that the spurious subgraph is usually \textit{larger} than the invariant across a variety of applications. For example, the biochemical property of molecular graphs is usually determined by a small functional group in a molecule~\citep{fragment}. On the other hand, the spurious subgraph such as the scaffold of the molecule usually takes a large part of the graph and easily biases the GNNs~\citep{ji2022drugood}.

In addition to the proof, we also conducted a empirical study using SPMotif datasets~\cite{wu2022discovering}, following the method proposed in \cite{kirichenko2023layer}. Specifically, we utilize the representations obtained through Eqn.~\ref{infomax_eq} to examine the extent to which these representations contain invariant features, as illustrated in Figure \ref{fig:emb_quality}: While the training accuracy is relatively high, there is a significant decrease in performance on the test set after feature reweighting (fine-tuning), compared to ERM. Specifically, the accuracy falls below 40\% on both datasets. This decline implies that the representations derived from Eqn.~\ref{infomax_eq} are predominantly composed of spurious features. More details about the experiments are included in Appendix~\ref{app:rep_quality}.

\begin{figure}[t]
    \centering
    \includegraphics[width=0.49\textwidth]{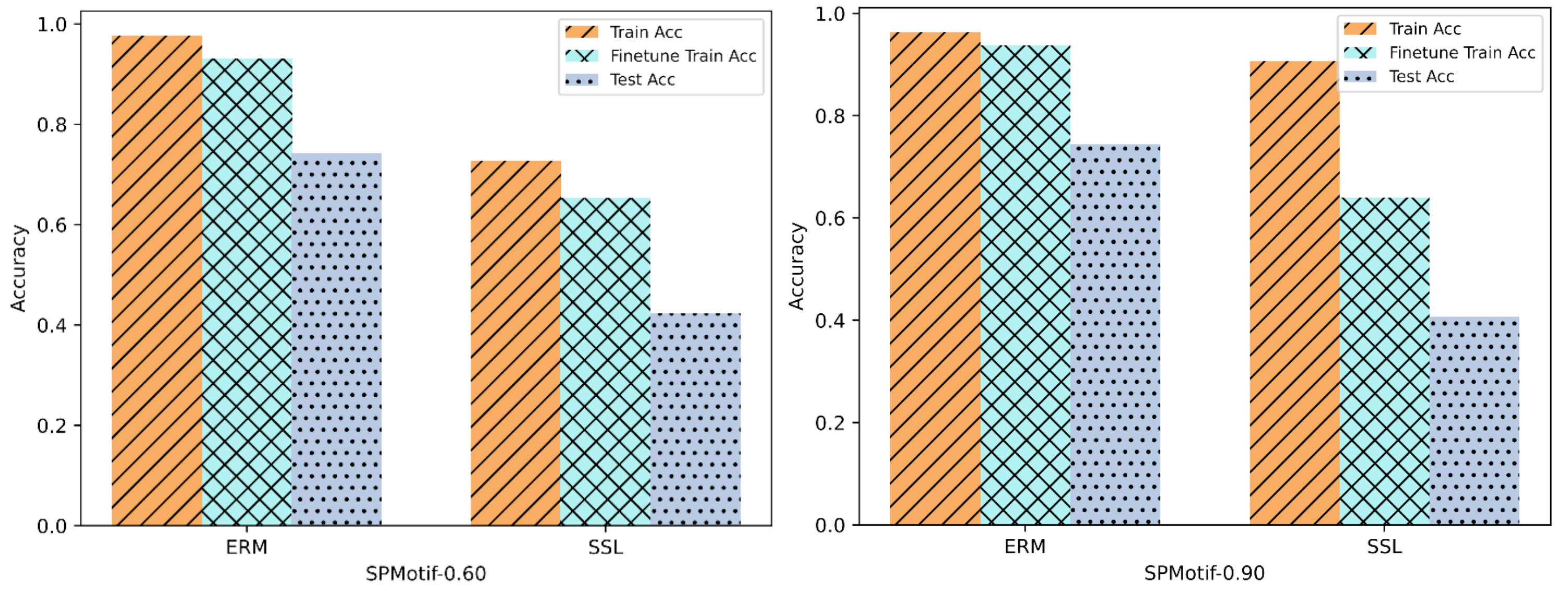}
    \vspace{-0.2in}
    \caption{Investigation of the representation quality of ERM and Infomax-Based SSL in capturing spurious features. The experimental results implies that Infomax-Based SSL primarily learns spurious correlations. Further details on the experimental setup are provided in Appendix~\ref{app:rep_quality}.}
    \label{fig:emb_quality}
    \vspace{-0.2in}
\end{figure}

\begin{figure*}[h]
\centering
\vspace{-1pt}
\scalebox{0.98}{
\includegraphics[width=0.95\textwidth]{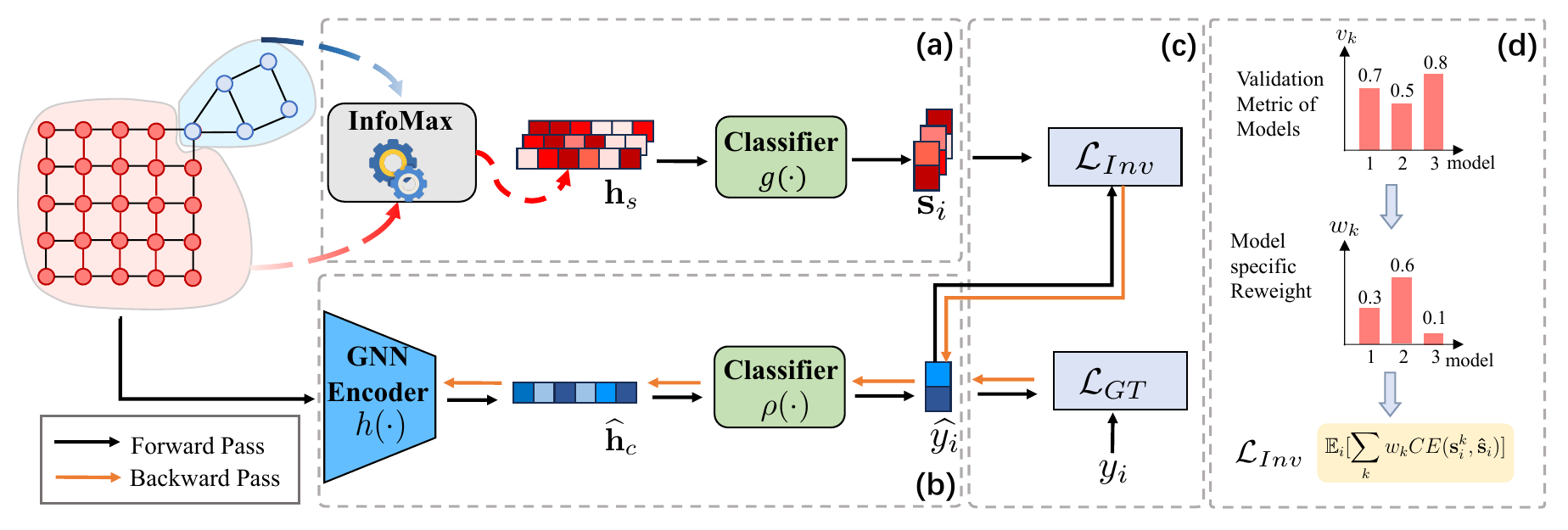}}
\caption{The overall framework of \ours. With an input graph consisting of $G_c$ (shown in blue) and $G_s$ (shown in red), the following procedures in \ours are illustrated: \textbf{(a)} Encoding and quantifying: First, the infomax-based SSL is performed to learn a collection of spurious representations (3 in this case), followed by $g(\cdot)$ to obtain the corresponding prediction logits as targets. \textbf{(b+c)} Decorrelation: In (b), a GNN encoder $h(\cdot)$ is re-trained from scratch to generate $\widehat{\vh}_c$ followed by a classifier $\rho(\cdot)$ to get the prediction $\widehat{y}_i$; In (c), $\widehat{y}_i$ is fed into loss function $\mathcal{L}=\mathcal{L}_{GT}+\lambda\mathcal{L}_{Inv}$ for learning invariant features, where $\widehat{\mathbf{s}}_i$ and $y_i$ serve as targets. \textbf{(d)} Detailed illustration of the process for model-specific reweighting. Finally, $\widehat{y}_i$ and $\widehat{\mathbf{s}}_i$ are obtained using the same classifier,i.e., $\rho(\cdot)=\rho^{\prime}(\cdot)$.}
\vspace{-0.1in}
\label{fig:equad}
\end{figure*}

\section{The \ourst Framework}
\label{sec:method}

So far, we have discussed how to obtain representations that solely comprise $\mathbf{h}_s$. Next, we introduce our proposed learning paradigm: \textit{Encoding, Quantifying, and Decorrelation}. This paradigm relies on the representations derived from infomax based self-supervised learning to achieve the decorrelation of $\mathbf{h}_s$ and $\widehat{\mathbf{h}}_c$, thereby obtaining invariant features for OOD generalization. The overall framework of \ours is illustrated in Figure~\ref{fig:equad}.

First, we introduce the following learning objective for subsequent discussion, whose optimal solution under both FIIF and PIIF, can elicit invariant representations.
\begin{equation}
\label{eq:prop1}
\max \; I\left(\widehat{\mathbf{h}}_c ; Y\right) \text {, s.t. } \widehat{\mathbf{h}}_{c} \ind E, \widehat{\mathbf{h}}_{c}=h(G), E \subseteq \mathcal{E}_{tr}.
\end{equation}
Eqn.~\ref{eq:prop1} is widely adopted in previous works~\cite{yang2022learning,chen2022learning,li2022learning,wu2022handling}, hence we omit the proof. To effectively solve Eqn. \ref{eq:prop1}, we first generate latent representations that maximally include $S$ with self-supervised learning. Then, we establish connections between $\widehat{\mathbf{h}}_{c}$ and $\mathbf{h}_{s}$, and transform $\mathbf{h}_{s}$ to a low-dimensional space. Finally, we leverage $\mathbf{h}_{s}$ (in the low-dimensional space) to recover $\mathbf{h}_{c}$ with our proposed  learning objective. Our approach is detailed as follows.

\textit{Step 1: Encoding.} In the first step, we utilize Eqn. \ref{infomax_eq} to train an encoder $h(\cdot)$ for generating representations that predominantly contain $\mathbf{h}_s$. However, due to potential optimization errors or model architectures, the representations learned may encompass only a subset of $\mathbf{h}_s$.  This limitation may impact the subsequent effectiveness of the decorrelation process. To mitigate this issue, we generate a collection of latent representations based on different training epochs and model architectures, i.e., $\mathcal{H}:=\left\{\mathbf{H}^{(i, j)} \in \mathbb{R}^{N \times F} \mid i \in [K], j \in \mathcal{P}\right\}$, aiming to comprehensively cover spurious features. Here, $\mathcal{P}$ denote a set of pre-defined epochs, and $K$ is the total number of model architectures, $N$ is the data sample size and $F$ is the embedding dimension. 

\textit{Step 2:Quantifying.} Having acquired a set of latent representations $\mathcal{H}$, we focus on the term $\widehat{\mathbf{h}}_c \ind E$ in Eqn. \ref{eq:prop1}. Assuming there exists an inverse and subjective function $f_{spu}^{-1}$ such that $E=\mathrm{f}_{s p u}^{-1}(S, C)$ under FIIF, and $E=\mathrm{f}_{s p u}^{-1}(Y, C)$ under PIIF. To fulfill the condition $\widehat{\mathbf{h}}_c \ind E$, we have the following optimization problem and its upper bound for FIIF:
\begin{equation}
\begin{aligned}
\min_{\widehat{C}} & \; I(E ; \widehat{C}) \\
= & I\left(f_{s p u}^{-1}(S, \widehat{C}) ; \widehat{C}\right) \\
\leq & I(S, \widehat{C} ; \widehat{C}) \\
= & I(S ; \widehat{C})+I(\widehat{C} ; \widehat{C} \mid S) \\
\leq & H(S)-H(S \mid \widehat{C}).
\end{aligned}
\end{equation}
The upper bound for PIIF can be derived in a similar manner. The first upper bound can be obtained as $f_{spu}^{-1}$ is a subjective function, and the second upper bound is due to that $I(\widehat{C} ; \widehat{C} \mid S)>0$. This upper bound is equivalent to the following optimization problem:
\begin{equation}
\label{max_opt_CS}
\max \; H\big(\mathbf{h}_s \mid \widehat{\mathbf{h}}_c\big).
\end{equation}
However, given that $\mathbf{h}_s$ are high-dimensional vectors, maximizing $H(\mathbf{h}_s \mid \widehat{\mathbf{h}}_c)$ remains a challenging task in practice. To make Eqn.~\ref{max_opt_CS} more tractable, we first transform $\mathbf{h}_s$ into a more compact representation while preserving essential information, through the following approach: We employ ground-truth labels $Y$ to train multiple classifiers $g: \mathbb{R}^F \rightarrow \mathbb{R}^{|\mathcal{C}|}$ (e.g., linear SVMs or MLPs) using $\mathcal{H}$ as inputs via ERM. Since $\mathcal{H}$ only contains a subset of $S$, $\mathbf{s}_i \in \mathbb{R}^{|\mathcal{C}|}$ can only depend on spurious features to make the prediction. Consequently, $\mathbf{s}_i$ can only reflect the correlation degree between the spurious pattern of sample $i$ and its corresponding label. Therefore, $\mathbf{s}_i$ can serve as a more compact representation of $\mathbf{h}_s$, also revealing side information about the training environments. With this quantification stage, we obtain the logits matrix $\mathcal{S}\!:=\left\{\mathbf{S}^{(i, j)} \in \mathbb{R}^{N \times \mathcal{|C|}} | \mathbf{S}^{(i, j)}=g(\mathbf{H}^{(i,j)},Y),i \in\ [K], j \in \mathcal{P}\right\}$. Figure~\ref{fig:logits_distribution} illustrates logits distribution for \textit{Cycle} class of SPMotif, which demonstrates the effectiveness of logits in identifying the spurious correlations for the data samples.

\begin{figure}[t]
    \centering
    \includegraphics[width=0.47\textwidth]{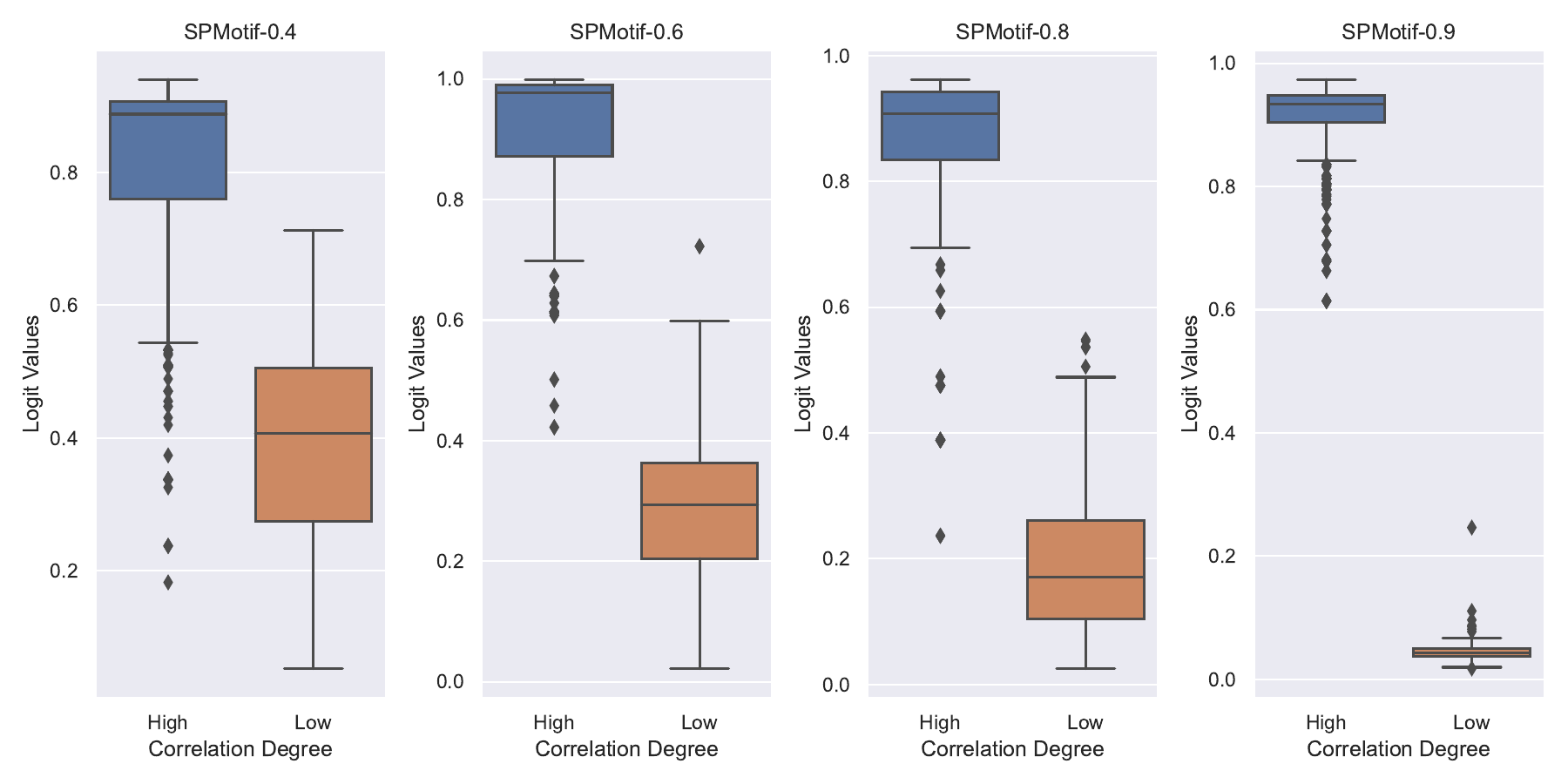}
    \vspace{-0.1in}
    \caption{Logits distribution of SPMotif datasets for the \textit{Cycle} class, where the samples are devided into two subgroups: the first subgroup consists of samples with a high correlation between $S$ and $Y$, and the second subgroup contains samples with a low correlation between $S$ and $Y$.}
    \label{fig:logits_distribution}
    \vspace{-0.1in}
\end{figure}

\textit{Step 3: Decorrelation.} Having obtained $\mathcal{S}$, we first formulate a learning objective to facilitate decorrelation of $\widehat{\mathbf{h}}_c$ and $\mathbf{h}_s$. We further refine this learning objective to mitigate the data imbalance problem, thus achieving better OOD generalization capability. Now our goal is to maximize the entropy term $H\big(\mathbf{s}_i \mid \widehat{\mathbf{h}}_c^{(i)}\big),\forall i \in [N]$ . To solve this problem, we propose the following learning objective:
\begin{equation}
\begin{aligned}
\label{inv_eq1}
&\underset{\rho, h}{\min} - \frac{1}{N} \; \sum_i \sum_j s_{ij} \log \big(\widehat{s}_{ij}\big), \\
&\text { s.t. } \widehat{\mathbf{s}}_i=\rho\big(\widehat{\mathbf{h}}_c^{(i)}\big)=\rho\left(h\left(G_i\right)\right), i \in [N], j \in [|\mathcal{C}|].
\end{aligned}
\vspace{-0.03in}
\end{equation}
Here, $\widehat{\mathbf{s}}_i$ is the estimated logits vector, and $s_{ij}$ is the scalar logit value for class $j$. $\mathbf{s}_i$ is the target normalized logits vector drawn from one of $\mathcal{S}$. In step 3, the encoder $h(\cdot)$ and classifier $\rho(\cdot)$ are trained from scratch to generate invariant representations. Here we offer an intuitive explanation of why Eqn. \ref{inv_eq1} achieve decorrelation of $\widehat{\mathbf{h}}_c$ and $\mathbf{h}_s$ by presenting a toy example as following.

\textit{Example.} Considering a set of positive data points \(\mathcal{D}_+\) with label \(y=1\), if $g(\cdot)$ divides \(\mathcal{D}_+\) into two equal-sized subsets: \(\mathcal{D}_+^1\) exploiting strongly correlated spurious patterns, and \(\mathcal{D}_+^2\) without such patterns. Assume \(\mathcal{D}_+^1\) has prediction logits of the form \((p, 1\!-\!p)\), e.g., \((0.9, 0.1)\), while \(\mathcal{D}_+^2\) has the inverse logits \((1\!-\!p, p)\), e.g., \((0.1, 0.9)\). It can be shown that the estimated logits \(\widehat{\mathbf{s}}_i\) minimizing Eqn. \ref{inv_eq1} is \((0.5, 0.5)\). In other words, $H\big(\mathbf{s}_{i} \! \mid \! \widehat{\mathbf{h}}_{c}^{(i)}\big)$ is maximized, regardless of the value of $p$. Formally, we present the following theorem:

\begin{theorem}
\label{theorem:inv_eq}
Let $s_{iy} \in \mathbb{R}_+$ denote the logits value for samples whose class label $Y=y$ and belonging to bin $B_i$, and $d(\cdot,\cdot)$ denote the distance function. Assuming that for each class label $Y=y$, there exists $2K$ bins $B_i \in \mathcal{B}, \forall i \in [2K]$ with equal sample size, furthermore, the $2K$ bins can be arranged into $K$ pairs of bins that are symmetrically located around the value $0.5$, i.e., $|s_{iy}|=|s_{(i+K)y}|,\forall i \in [K]$, and $ d(\mathbb{E}[s_{iy}],0.5)=d(\mathbb{E}[s_{(i+K)y}],0.5)$, where $sign(\mathbb{E}[s_{iy}]-0.5)=sign(0.5-\mathbb{E} [s_{(i+K)y}])$. Under these conditions, Eqn. \ref{inv_eq1}, serving as a penalty term for ERM, achieves unique and optimal solution, i.e., $\widehat{\mathbf{h}}_c=\mathbf{h}_c$ and $\widehat{\mathbf{h}}_c \ind \mathbf{h}_s$.
\end{theorem} 
\vspace{-0.1in}

A formal proof is provided in Appendix~\ref{proof:theo_inv_eq}, where we prove for both necessity and sufficiency for the optimal solution for Eqn.~\ref{inv_eq1}.

\textit{Mitigating data imbalance.} Theorem \ref{theorem:inv_eq} assumes an equal distribution of samples across all bins for every label $Y$. However, in real-world scenarios, certain spurious patterns, which are highly correlated with $Y$, often dominate, leading to a disproportionate accumulation of samples in specific bins. This data imbalance can hurt the optimality of Eqn. \ref{inv_eq1}. To address this, we propose a novel sample reweighting approach that increases the weight of the minority group to achieve a more balanced distribution of data across different levels of correlation under the same target label $Y$, thereby facilitating the decorrelation of $\widehat{\mathbf{h}}_c$ and $\mathbf{h}_s$. Specifically, we adopt the following reweighting function: $w_{iy}:=w\left(s_{iy} ; \gamma\right)=\frac{1-s_{iy}^\gamma}{\gamma}, \; s.t. \; \forall i \in [N], Y(i)=y$, where $\gamma$ is a hyperparameter to control the smoothness of the function. For all samples associated with the same $y$, $w(\cdot)$ will assign greater weight to samples where $s_{iy}$ is closer to zero, indicating that these are minority samples whose spurious features less frequently co-occur with the prediction. Moreover, considering that a single prediction logits matrix $\mathbf{S} \subset \mathcal{S}$ might only capture a subset of the spurious features, thus limiting the identification of $\mathbf{h}_s$, we draw multiple logits matrices from $\mathcal{S}$ for better prediction of correlation degree of spurious features. Finally, we introduce a model-specific reweighting strategy to direct the loss term towards higher-quality target logits vectors $\mathbf{s}_i^k, \forall k \in [K]$, where $K$ represents the number of prediction logits matrices. The quality of $\mathbf{S}_k$ is assessed based on the validation metric: \textit{A lower validation metric indicates a reduced effectiveness of invariant features, thereby more accurately reflecting the correlation degree of spurious patterns.} Specifically, we adopt temperature-scaled softmax function for model-specific reweighting, i.e., $w_k:=m(v_k;\tau)=\frac{\exp \left(\frac{-v_k}{\tau}\right)}{\sum_{j=1}^K \exp \left(\frac{-v_j}{\tau}\right)}$, where $\tau$ is the temperature, and $v_k$ is the validation metric from model $k$. Finally, we arrive at the following refined objective:
\begin{equation}
\label{inv_eq2}
\begin{aligned}
& \min _{\rho, h}-\frac{1}{N K} \sum_i \sum_j \sum_k w_{ij} w_{k} s_{i j}^k \log \left(\widehat{s}_{i j}\right), \\
& \text { s.t. } \widehat{\mathbf{s}}_i\!=\!\rho\left(\widehat{\mathbf{h}}_c^{(i)}\right)\!=\!\rho\left(h\left(G_i\right)\right)\!, i \in[N], j \in[|C|], k \in[K].
\end{aligned}
\end{equation}
Here $w_{ij}$ and $w_k$ are the sample-specific and model-specific reweighting coefficients respectively. Finally, to obtain the $K$ prediction logits matrices from $\mathcal{S}$, we can also use the validation metric (e.g., validation accuracy) as a measure to gauge the extent of spurious features in the representations. Let $\mathcal{L}_{Inv}(\rho,h)$ denote Eqn.~\ref{inv_eq2}, and let $\mathcal{L}_{GT}$ denote the supervised training loss, i.e., $\mathcal{L}_{GT}(\rho^\prime,h)=\sum_{i=1}^N CE(y_i,\widehat{y_i})$, where $CE$ denote the cross-entropy loss for ground-truth label $y_i$ and predicted class label $\widehat{y_i}$. In the specific implementation, $\rho(\cdot)$ and $\rho^{\prime}(\cdot)$ share the same model parameters. Finally, the loss function for EQuAD is:
\begin{equation}
\label{final_loss}
\mathcal{L}=\mathcal{L}_{GT}+\lambda \mathcal{L}_{Inv},
\end{equation}
where $\lambda$ controls the strength of the decorrelation loss $\mathcal{L}_{Inv}$.

\section{Experiments}
\label{sec:exp}

In this section, we conduct extensive experiments to answer the following research questions. 


\vspace{-0.1in}
\begin{itemize}
    \item \textbf{RQ1)} Does EQuAD achieve better or comparable predictive performance than state-of-the-art methods?
    \item \textbf{RQ2)} How can we examine and interpret the latent representation induced by the GNN encoder in step 3?
    \item \textbf{RQ3)} How does each component in EQuAD contributes to the final performance?
\end{itemize}
\vspace{-0.1in}

\begin{table*}[t]
\centering
\caption{Experiment results on synthetic datasets.}
\scalebox{0.92}{
\begin{tabular}{@{}p{1.5cm}|llll|lll@{}}
\toprule
\multicolumn{1}{c|}{\multirow{2}{*}{Methods}} & \multicolumn{4}{c|}{SPMotif}                       & \multicolumn{3}{c}{Two-piece graph}  \\ \cmidrule(l){2-8} 
\multicolumn{1}{c|}{} &
  \multicolumn{1}{c}{$b=0.33$} &
  \multicolumn{1}{c}{$b=0.40$} &
  \multicolumn{1}{c}{$b=0.60$} &
  \multicolumn{1}{c|}{$b=0.90$} &
  \multicolumn{1}{c}{$(0.80,0.70)$} &
  \multicolumn{1}{c}{$(0.80,0.90)$} &
  \multicolumn{1}{c}{$(0.70,0.90)$} \\ \midrule
ERM                                           & 53.40±2.20  & 62.19±3.26 & 55.24±2.43 & 49.41±3.78 & 75.65±1.62 & 51.37±1.20 & 42.73±3.82 \\
IRM                                           & 58.31±2.59  & 55.71±6.37 & 58.76±1.98 & 42.11±4.14 & 75.13±0.77 & 50.76±2.56 & 41.32±2.50 \\
V-Rex                                         & 56.12±4.76  & 60.08±4.11 & 58.91±4.45 & 42.32±3.48 & 74.96±1.40 & 49.47±3.36 & 41.65±2.78 \\
IB-IRM                                        & 60.96±3.19  & 57.52±4.84 & 58.51±3.57 & 47.01±4.07 & 73.93±0.79 & 50.93±1.87 & 42.05±0.79 \\
EIIL                                          & 59.87±2.19  & 57.73±5.09 & 53.42±3.84 & 42.58±5.42 & 74.25±1.74 & 51.45±4.92 & 39.71±2.64 \\ \midrule
GREA                                          & 59.27±3.45  & 62.46±4.28 & 61.04±5.21 & 58.63±1.52 & 82.72±0.50 & 50.34±1.74 & 39.01±1.21 \\
GSAT                                          & 52.48±6.55  & 60.17±3.42 & 60.42±3.08 & 56.22±5.84 & 78.11±1.23 & 48.63±2.18 & 36.62±0.87 \\
GIL                                           & 57.92±5.03  & 65.34±3.24 & 58.86±7.25 & 57.09±7.33 & 82.67±1.18 & 51.76±4.32 & 40.07±2.61 \\
DisC                                          & 49.79±6.01  & 55.22±4.75 & 47.22±8.97 & 50.51±4.39 & 54.29±15.0 & 45.06±7.82 & 39.42±8.59 \\
CIGA                                          & 72.91±1.92  & 67.96±5.27 & 67.31±6.84 & 58.87±5.93 & 83.21±0.30 & 57.87±3.38 & 43.62±3.20 \\
GALA                                          & 66.96±5.18 & 65.38±3.68 & 63.25±3.11 & 62.07±2.20 & \textbf{83.65±0.44} & 62.25±3.71 & 49.65±3.93 \\ \midrule
EQuAD &
  \textbf{74.61±1.23} &
  \textbf{73.13±1.56} &
  \textbf{71.93±1.94} &
  \textbf{69.47±2.06} &
  82.76±0.71 &
  \textbf{75.81±0.51} &
  \textbf{71.95±1.41} \\ \bottomrule
\end{tabular}}
\label{tab:spmotif}
\end{table*}

\begin{figure}[t]
\centering
\scalebox{0.9}{
\includegraphics[width=0.48\textwidth]{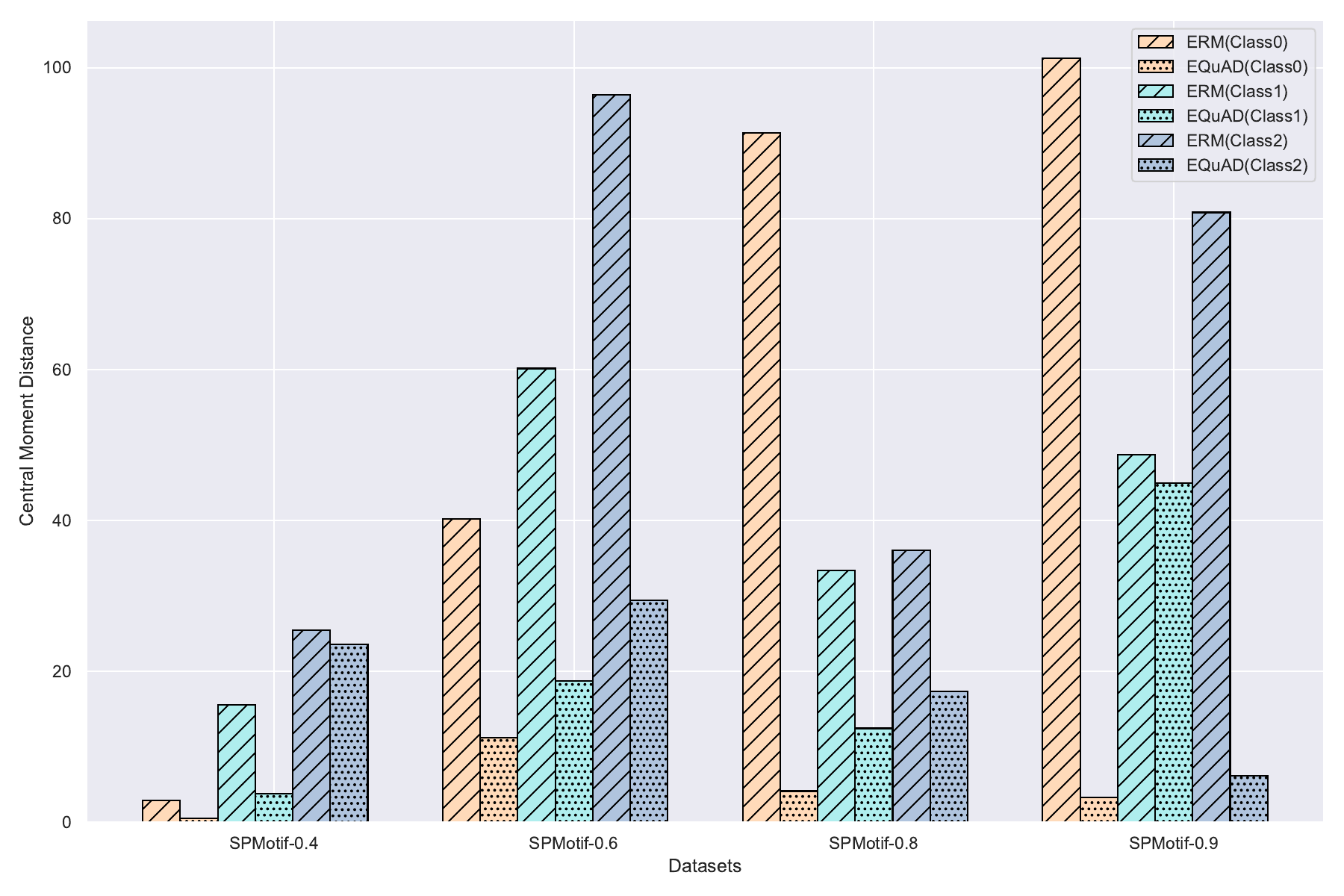}}
\vspace{-5pt}
\caption{Investigation of the central moment distance of latent embeddings from training and validation set respectively using ERM and EQuAD. For all three classes across various SPMotif datasets, the distances for representations obtained via EQuAD are notably smaller compared to those from ERM.} 
\label{fig:cmd}
\vspace{-10pt}
\end{figure}

\subsection{Experimental Setting}

\textbf{Datasets.} To comprehensively evaluate our proposed method under two data generating assumptions, namely FIIF and PIIF, we utilize the SPMotif datasets~\cite{wu2022discovering} and Two-piece graph datasets \cite{chen2023does} to verify its effectiveness. Additionally, for real-world datasets, we employ the DrugOOD datasets~\cite{ji2022drugood}, which focus on the challenging task of AI-aided drug affinity prediction. We adopt 6 DrugOOD subsets, including splits using Assay, Scaffold, and Size from the IC50 and EC50 category respectively. Moreover, we also consider two molecule datasets, MolBACE and MolBBBP, from Open Graph Benchmark \cite{hu2021open}, where different molecules are structurally separated into different subsets, which provides a more realistic estimate of model performance in experiments \cite{wu2018moleculenet}. More details about datasets are included in Appendix \ref{app:datasets}. 

\textbf{Baseline methods.} Besides ERM \cite{vapnik95}, we compare our method with state-of-the-art OOD methods from the Euclidean regime, including IRM \cite{arjovsky2020invariant}, VREx \cite{krueger2021outofdistribution}, EIIL \cite{creager2021environment}, IB-IRM \cite{ahuja2022invariance}, Coral \cite{sun2016deep} and MixUp \cite{zhang2018mixup}. For graph-specific algorithms, we include GREA \cite{Liu_2022}, DIR \cite{wu2022discovering}, GSAT \cite{miao2022interpretable}, CAL \cite{Sui_2022}, DisC \cite{fan2022debiasing}, MoleOOD \cite{yang2022learning}, GIL \cite{li2022learning}, CIGA \cite{chen2022learning}, GALA \cite{chen2023does} and iMoLD \cite{zhuang2023learning} as strong competitive baseline methods. For all baseline methods and EQuAD, we use GIN \cite{xu2019powerful} as backbone encoder, and use Adam \cite{kingma2017adam} as the optimizer, for fair comparisons.  

\textbf{Evaluation.} For SPMotif datasets and Two-piece graph datasets, the task is a 3-class classification, we adopt \textit{accuracy} as the evaluation metric. For DrugOOD datasets and the two molecular datasets, we perform binary classification using \textit{AUC} as the evaluation metric. To investigate the distribution discrepancy of two sets of embeddings, we adopt \textit{central moment distance} \cite{zellinger2019central} as a quantitative measure.

\subsection{Main Results (RQ1)}

\begin{table*}[t]
\centering
\caption{Experiment results on real-world datasets. \ours outperforms state-of-the-art methods in most of the datasets with various types of distribution shifts. First and second best methods are denoted by \textbf{bold} and \underline{underline} respectively.}
\scalebox{0.92}{
\begin{tabular}{@{}p{1.5cm}cccccccc@{}}
\toprule
\multirow{2}{*}{Method} & \multicolumn{3}{c}{IC50} & \multicolumn{3}{c}{EC50} & \multirow{2}{*}{BACE} & \multirow{2}{*}{BBBP} \\ \cmidrule(lr){2-7}
        & Assay       & Scaffold    & Size        & Assay       & Scaffold    & Size        &   & \multicolumn{1}{l}{} \\ \midrule
ERM     & 71.63±0.76 & 68.79±0.47 & 67.50±0.38 & 67.39±2.90 & 64.98±1.29 & 65.10±0.38 & 77.83±3.49 & 66.93±2.31                    \\
IRM     & 71.15±0.57 & 67.22±0.62 & 61.58±0.58 & 67.77±2.71 & 63.86±1.36 & 59.19±0.83 & 79.47±1.86 & 68.92±0.53                    \\
Coral   & 71.28±0.91 & 68.36±0.61 & 64.53±0.32 & 72.08±2.80 & 64.83±1.64 & 58.47±0.43 & - & -                    \\
MixUp   & 71.49±1.08 & 68.59±0.27 & 67.79±0.39 & 67.81±4.06 & 65.77±1.83 & 65.77±0.60 & - & -                    \\ 
V-Rex   & 71.32±1.17 & 67.30±1.27 & 64.46±0.79 & 75.57±2.17 & 64.73±0.53 & 62.80±0.89 & - & -                    \\
IB-IRM   & 68.22±0.54 & 59.38±0.49 & 58.25±2.40 & 64.70±2.50 & 62.62±2.05 & 58.28±0.99 & - & -                    \\
EIIL     & 70.58±1.56 & 67.02±0.46 & 61.58±0.58 & 64.20±5.40 & 62.88±2.75 & 59.58±0.96 & - & -                    \\ \midrule
DIR     & 69.84±1.41 & 66.33±0.65 & 62.92±1.89 & 65.81±2.93 & 63.76±3.22 & 61.56±4.23 & 79.93±2.03 & 69.63±1.54                    \\
GSAT    & 70.59±0.43 & 66.45±0.50 & 66.70±0.37 & 73.82±2.62 & 64.25±0.63 & 62.65±1.79 & 79.63±1.87 & 68.48±2.01                    \\
GREA    & 70.23±1.17 & 67.02±0.28 & 66.59±0.56 & 74.17±1.47 & 64.50±0.78 & 62.81±1.54 & \textbf{82.37±2.37} & 69.70±1.28                    \\
CAL     & 70.09±1.03 & 65.90±1.04 & 66.42±0.50 & 74.54±4.18 & 65.19±0.87 & 61.21±1.76 & - & -                    \\
DisC    & 61.40±2.56 & 62.70±2.11 & 61.43±1.06 & 63.71±5.56 & 60.57±2.27 & 57.38±2.48 & - & -                    \\
MoleOOD & 71.62±0.52 & 68.58±1.14 & 65.62±0.77 & 72.69±1.46 & 65.74±1.47 & 65.51±1.24 & 81.09±2.03  & \underline{69.84±1.84}                     \\
CIGA    & 71.86±1.37 & \underline{69.14±0.70} & 66.92±0.54 & 69.15±5.79 & 67.32±1.35 & 65.65±0.82 & 80.98±1.25  & 69.65±1.32                     \\ 
iMoLD    & \underline{72.11±0.51} & 68.84±0.58 & \underline{67.92±0.43} & \underline{77.48±1.70} & \underline{67.79±0.88} & \textbf{67.09±0.91} & -  & -                     \\ \midrule
EQuAD   & \textbf{73.26±0.47} & \textbf{69.78±0.41} & \textbf{68.19±0.24} & \textbf{79.36±0.73} & \textbf{68.12±0.48} & \underline{66.37±0.64} & \underline{81.83±2.67} & \textbf{71.22±1.47}                    \\ \bottomrule
\end{tabular}}
\vspace{-0.05in}
\label{tab:realworld_perf}
\end{table*}

\textbf{Synthetic datasets.} Our experimental results on two synthetic datasets are reported in Table \ref{tab:spmotif}. EQuAD demonstrates superior performance on these datasets across varying degrees of bias. The results from SPMotif indicate that EQuAD maintains stable performance under different levels of spurious correlation, consistently outperforming other baseline methods. In the context of PIIF, particularly when $H(S|Y)\!\!<\!\!H(C|Y)$, e.g., $(\alpha=0.8,\beta=0.9)$ and $(\alpha=0.7,\beta=0.9)$, environment inference algorithms(e.g., MoleOOD \cite{yang2022learning} and GIL \cite{li2022learning}) and environment augmentation algorithms (e.g., GREA \cite{Liu_2022} and DIR \cite{wu2022discovering}) all fail catastrophically. When the correlation between $S$ and $Y$ strengthens, these algorithms inevitably learn $C$ in their representations, failing to accurately isolate $S$, which adversely affects subsequent invariance learning. Moreover, when $H(S|Y) < H(C|Y)$, CIGA's \cite{chen2022learning} objective fails to correctly identify $G_c$, and GSAT \cite{miao2022interpretable}, based on the Information Bottleneck principle \cite{tishby2015deep}, discards information about $C$ as $S$ reveals more information about $Y$. The key to EQuAD's success is its ability to identify $S$ through self-supervised global-local mutual information maximization, which does not depend on label $Y$. Therefore, the strong association between $S$ and $Y$ does not impede the identification of $S$. It is noteworthy that GALA \cite{chen2023does} employs data sampling for spurious subgraphs to address data imbalance. In our experiments, we exclude the data sampling technique for GALA, and similarly, EQuAD is tested without this strategy to ensure fair comparisons.

\textbf{Real-world datasets.} The effectiveness of EQuAD is further demonstrated through its performance on real-world datasets, as in Table \ref{tab:realworld_perf}. In these practical scenarios, EQuAD consistently outperforms other methods, achieving state-of-the-art results across most datasets. Notably, EQuAD shows significant improvements over both environment inference and environment augmentation algorithms in almost all datasets. This indicates that our approach, leveraging infomax based self-supervision, is more adept at identifying $S$, thereby facilitating the learning of invariant features. Furthermore, EQuAD consistently outperforms CIGA, which is known to reliably identify $G_c$ when $H(C|Y) < H(S|Y)$. This highlights the effectiveness of EQuAD in directly learning invariant representations from the latent space. DisC also adopt a similar idea as EQuAD, i.e., to first identity spurious features ( biased graphs), and then utilize the spurious features for the desired invariant features (learn invariant features on debiased graphs), however, DisC~\cite{fan2022debiasing} relies on the strong correlation between $S$ and $Y$ to learn the biased graphs, which hinders its performance, while EQuAD resort to self-supervision to reliably extract spurious patterns which is orthogonal to $\mathbb{P}(S,Y)$. Although iMoLD~\cite{zhuang2023learning} also employs a self-supervised objective for invariance learning to encourage the separation of $S$ and $C$, their approach still relies on a parametric model with labeled supervision to obtain $\widehat{S}$. Consequently, $\widehat{S}$ may retain a subset of $C$, potentially compromising the efficacy of invariant feature learning.

\subsection{Visualization (RQ2)}

This section delves into the quality of the latent representation derived from the encoder $h(\cdot)$ in step 3. We explore the distribution discrepancy between latent embeddings from the training set and those from the validation set. More visualizations on latent representation with t-SNE\cite{JMLR:v9:vandermaaten08a} plots in a 2-dimensional space are provided in Appendix~\ref{app:exp}. 

\begin{figure}[t]
\centering
\vspace{-1pt}
\scalebox{0.9}{
\includegraphics[width=0.48\textwidth]{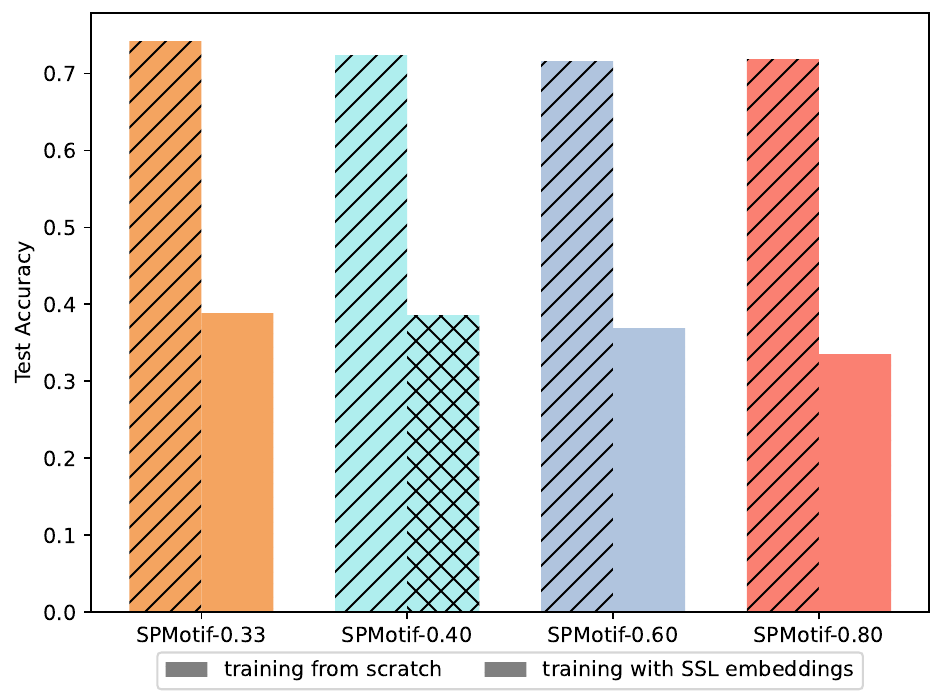}}
\vspace{-10pt}
\caption{Ablation study on latent representations, which demonstrates the necessity of training a model $f(\cdot)$ from scratch for effective invariance learning.} 
\label{fig:ssl_testAcc}
\vspace{-0.15in}
\end{figure}

\textbf{Analysis of Distribution Discrepancy in Latent Representations.} We investigate the discrepancy in the distribution of latent representations between the training and validation sets. To quantify this discrepancy, we employ \textit{the central moment distance} (CMD), a metric proposed in \cite{zellinger2019central} for domain-invariant representation learning in neural networks. As illustrated in Figure \ref{fig:cmd}, across all three ground-truth labels, the representations obtained through EQuAD exhibited relatively smaller distribution discrepancies between training and validation sets compared to those obtained through ERM. This suggests EQuAD can achieve superior OOD generalization capabilities. 

\begin{table}[t]
\centering
\caption{Ablation study on model-specfic reweighting.}
\resizebox{0.48\textwidth}{!}{%
\begin{tabular}{@{}p{2.0cm}lll@{}}
\toprule
            & EC50-Assay  & EC50-Scaffold & EC50-Size   \\ \midrule
w/ weighted & \textbf{79.36±0.63}  & \textbf{68.12±0.48}    & \textbf{66.37±0.64}  \\
w/ averaging & 78.03±2.24  & 67.45±0.32    & 66.13±0.38  \\ 
w/ single & 77.46±0.29  & 67.68±0.17    & 66.02±0.41  \\  \bottomrule
\end{tabular}
}
\label{tab:abstudy_weighted}
\vspace{-0.2in}
\end{table}

\subsection{Ablation Study (RQ3)}

In this section, we conduct two ablation study to further validate our design choice in EQuAD.

\textbf{Adaptive Reweighting vs. Averaging}. We perform ablation study to investigate the effectiveness of using multiple prediction logits matrices from $\mathcal{S}$ compared to a single logits matrix. Additionally, we explore the benefits of model-specific adaptive reweighting over simple averaging. As reported in Table \ref{tab:abstudy_weighted}, employing multiple logits matrices consistently outperform the use of a single logits matrix, which may be due to the different spurious patterns learned by distinct logits matrices by $g(\cdot)$, enabling a more accurate estimation of the correlation degree between $S$ and label $Y$. Furthermore, leveraging feedback based on validation metrics allows for a preference towards logits matrices of higher quality, thereby yielding improved empirical results.

\textbf{The Necessity of Training an Encoder from Scratch.}  As observed in Figure \ref{fig:ssl_testAcc}, embeddings derived from infomax-based SSL struggles to extract high-quality invariant features and results in low test accuracy, even trained with decorrelation loss. In contrast, training an encoder $f(\cdot)$ from scratch can extract higher-quality invariant representations and lead to better graph invariance learning. This again indicates that infomax-based SSL primarily learns spurious features, and highlights the effectiveness of our proposed learning paradigm, which learns both invariant features and spurious features via ERM~\cite{kirichenko2023layer}, and remove spurious features in step 3 which are learned from SSL.

\section{Conclusions}

We have shown that the infomax principle effectively extracts spurious features with provable guarantees, which motivates us to design a novel learning paradigm and the development of a flexible framework, \ours. \ours induces a robust inductive bias by eliminating the reliance on strong assumptions about the correlation strengths between spurious features and class labels. Our approach significantly outperforms state-of-the-art methods in most synthetic and real-world datasets, demonstrating the effectiveness and robustness of \ours for OOD generalization on graphs. This learning paradigm holds great potential for adaptation to other data modalities, such as vision and natural language, which we leave for our future work.

\section*{Impact Statement}

This paper is dedicated to advancing research in OOD Generalization, with a specific focus on graph data. The learning paradigm we propose holds the potential for broad application beyond graph data, extending to other data forms such as vision and natural language.

In practical applications, distribution shifts between testing and training data are inevitable, and traditional graph machine learning approaches often suffer from significant performance degradation under such shifts. Developing methods with robust OOD generalization capabilities is thus critically important, especially for high-stakes graph applications, including but not limited to molecule prediction~\cite{hu2021open}, financial analysis~\cite{ijcai2020p643}, criminal justice~\cite{agarwal2021unified}, autonomous driving~\cite{liang2020enhancing}, and drug discovery~\cite{gaudelet2021utilizing}. The ability to maintain performance consistency across varying distributions can have profound implications in these areas, potentially leading to more reliable predictions and analyses that can inform decision-making processes and contribute to advancements in these fields. Besides, this paper does not raise any ethical concern or potentially harmful insight.

\section*{Acknowledgements}

KZ would like to acknowledge the support from NSF Grant 2229881, the National Institutes of Health (NIH) under Contract R01HL159805, and grants from Apple Inc., KDDI Research Inc., Quris AI, and Florin Court Capital. We sincerely appreciate the insightful discussions with the anonymous reviewers, which greatly helped to improve our work.

\bibliography{example_paper}
\bibliographystyle{icml2024}

\newpage
\appendix
\onecolumn

\section{Notations}
\label{app:notations}
We present a set of notations used throughout our paper for clarity. Below are the main notations along with their definitions.

\begin{table*}[ht]

\centering
\caption{Notation Table}
\begin{tabular}{c|l}
\hline
\textbf{Symbols} & \textbf{Definitions} \\
\hline
$\mathcal{G}$ & a set of graphs \\
$G=(A, X)$ & a graph with the adjacency matrix $A \in\{0,1\}^{n \times n}$ and node feature matrix $X \in \mathbb{R}^{n \times d}$ \\
$Y$ & random variable for labels \\
$C$ & content factor \\
$S$ & style factor \\
$E$ & environment \\
$\mathbf{h}_c$ & invariant representations, interchangeably with $C$ in our paper \\
$\mathbf{h}_s$ & spurious representations, interchangeably with $S$ in our paper \\
$G_c$ & the invariant subgraph with respect to $G$  \\
$G_s$ & the spurious subgraph with respect to $G$ \\ \hline
$\widehat{G}_c$ & the estimated invariant subgraph \\
$\widehat{G}_s$ & the estimated spurious subgraph \\
$\widehat{\mathbf{h}}_G$ & the estimated graph representation for graph $G$ \\
$\widehat{\mathbf{h}}_i$ & the estimated node representation for node $i \in G$ \\
$\widehat{\mathbf{h}}_c$ & the estimated invariant graph representation, interchangeably with $\widehat{C}$ in our paper \\
$\widehat{\mathbf{h}}_s$ & the estimated spurious graph representation, interchangeably with $\widehat{S}$ in our paper \\
$\widehat{C}$ & the estimated content variable  \\
$\widehat{S}$ &  the estimated style variable  \\ \hline
$\mathbf{H}$ &  latent representations derived from global-local mutual information maximization  \\
$\mathbf{S}$ &  prediction logits matrix derived from $\mathbf{H}$ and $Y$ \\
$\mathcal{H}$ &  a set of $\mathbf{H}$  \\
$\mathcal{S}$ &  a set of $\mathbf{S}$  \\ \hline 
$g(\cdot)$ &  classifier used in step 2 to generate spurious logits matrix $\mathbf{S}$ \\
$h(\cdot)$ &  encoder  \\
$\rho(\cdot)$ &  classifier  \\ \hline
$[K]:=\left\{1,2,\cdots,K\right\}$ &  index set with $K$ elements  \\
$w$ &  a scalar value  \\
$\mathbf{w}$ &  a vector  \\
$\mathbf{W}$ &  a matrix  \\
\hline
\end{tabular}
\label{tab:notation}
\end{table*}

\section{More Details on Data Generating Process}
\label{app:data_gen}

We detail about the underlying assumption of the data generating process in our work as illustrated in Figure~\ref{fig:scm}. The graph generation process is illustrated using a latent-variable model perspective to elucidate, positing that a graph is generated through a mapping $f_{\text{gen}}: \mathcal{Z} \rightarrow \mathcal{G}$, where $\mathcal{Z} \subseteq \mathbb{R}^n$ represents the latent space and $\mathcal{G}$ denotes the graph space. Within this framework, we distinguish the latent variables into an content variable $C \in \mathcal{C}$ and a style variable $S \in \mathcal{S}$, based on their susceptibility to environmental influences $E$. The invariant and spurious components, $C$ and $S$, respectively, govern the observed graphs' generation, with their interactions in the latent space leading to the emergence of Fully Informative Invariant Features (FIIF) and Partially Informative Invariant Features (PIIF), depends on the completeness of information $C$ provides about the label $Y$.

The graph generation model is formalized through a Structural Causal Model (SCM) that decomposes $f_{\text{gen}}$ into distinct functions controlling the generation of $G_c$, $G_s$, and $G$, as outlined in Eqn.~\ref{eq:scm}. This decomposition allows for the isolation of invariant information within $G_c$, unaffected by environmental interventions, from the spurious and environment-sensitive information within $G_s$ and $G$. Such a model reflects the reality that graphs from different domains may exhibit diverse structural and feature-level properties, all potentially spuriously correlated with labels depending on the underlying latent interactions. The SCM framework, compatible with a broad range of graph generation models, aims to characterize potential distribution shifts without committing to specific graph families, thereby maintaining generality and applicability across various contexts. 

\begin{equation}
\label{eq:scm}
G_c:=f_{\text {gen }}^{G_c}(C), \quad G_s:=f_{\text {gen }}^{G_s}(S), \quad G:=f_{\text {gen }}^G\left(G_c, G_s\right)
\end{equation}

Following previous work~\cite{arjovsky2020invariant,ahuja2022invariance,chen2022learning}, the latent interactions between $C$ and $S$ are categorized into Fully Informative Invariant Features (FIIF) and Partially Informative Invariant Features (PIIF) based on whether $C$ fully informs about $Y$. In the context of FIIF (Eqn.~\ref{fiif}), the invariant component $C$ provides a complete and direct mapping to the label $Y$, making it fully informative for label prediction. This implies that, regardless of environmental variations or the presence of spurious correlations, a model can rely solely on the invariant features within $C$ for accurate predictions. This scenario is ideal for OOD generalization, as it suggests that the core attributes necessary for classification are consistent across different domains. Previous OOD methods, such as GIB~\cite{yu2020graph} and DIR~\cite{wu2022discovering} and GSAT~\cite{miao2022interpretable}, have primarily focused on leveraging FIIF to ensure that models remain robust against distribution shifts by concentrating on invariant features that are not affected by environmental changes. 

\begin{equation}
\label{fiif}
Y:=f_{\mathrm{inv}}(C), S:=f_{\mathrm{spu}}(C, E), G:=f_{\mathrm{gen}}(C, S) .
\end{equation}

PIIF (Eqn.~\ref{piif}), on the other hand, acknowledge that $C$ only partially informs about $Y$. This necessitates the consideration of additional variables, potentially including the spurious component $S$ or the environment $E$, to achieve accurate label prediction. In these cases, the model must discern how to integrate information from both invariant and spurious features, navigating the complex interactions that may indirectly influence the prediction of $Y$. The challenge here lies in identifying and mitigating the impact of spurious correlations that may be informative in a specific context but detrimental to generalization across environments. Methods like IRM~\cite{arjovsky2020invariant} focus on PIIF scenarios by attempting to learn representations that are predictive of $Y$ across different environments, despite the presence of varying spurious correlations.

\begin{equation}
\label{piif}
Y:=f_{\mathrm{inv}}(C), S:=f_{\mathrm{spu}}(Y, E), G:=f_{\mathrm{gen}}(C, S) .
\end{equation}

In this study, we provide a new perspective on how to learn graph invariance under both of these scenarios, and our proposed framework exhibits robust and superior performance under both scenarios with varying degrees of correlation strengths between $S$ and $Y$.

\section{More Background and Related Work}

\textbf{Graph Neural Networks.} In this work, we adopt message-passing GNNs for graph classification due to their expressiveness. Given a simple and undirected graph $G=(A,X)$ with $n$ nodes and $m$ edges, where $A \in\{0,1\}^{n \times n}$ is the adjacency matrix, and $X \in \mathbb{R}^{n \times d}$ is the node feature matrix with $d$ feature dimensions, the graph encoder $h: \mathbb{G} \rightarrow \mathbb{R}^h$ aims to learn a meaningful graph-level representation $h_G$, and the classifier $\rho: \mathbb{R}^h \rightarrow \mathbb{Y}$ is used to predict the graph label $\widehat{Y}_G=\rho (h_G)$. To obtain the graph representation $h_G$, the representation  \( \mathbf{h}_{v}^{(l)} \) of each node \( v \) in a graph $G$ is iteratively updated by aggregating information from its neighbors \( \mathcal{N}(v) \). For the \( l \)-th layer, the updated representation is obtained via an AGGREGATE operation followed by an UPDATE operation:

\begin{align}
\mathbf{m}_{v}^{(l)} &= \text{AGGREGATE}^{(l)}\left(\left\{\mathbf{h}_{u}^{(l-1)}: u \in \mathcal{N}(v)\right\}\right), \\
\mathbf{h}_{v}^{(l)} &= \text{UPDATE}^{(l)}\left(\mathbf{h}_{v}^{(l-1)}, \mathbf{m}_{v}^{(l)}\right),
\end{align}
where \( \mathbf{h}_{v}^{(0)} = \mathbf{x}_{v} \) is the initial node feature of node \( v \) in graph \( G \). Then GNNs employ a READOUT function to aggregate the final layer node features \( \left\{\mathbf{h}_{v}^{(L)} : v \in \mathcal{V}\right\} \) into a graph-level representation \( \mathbf{h}_{G} \):
\begin{equation}
\label{graph-rep Eq}
\mathbf{h}_{G} = \operatorname{READOUT}\left(\left\{\mathbf{h}_{v}^{(L)}: v \in \mathcal{V}\right\}\right).
\end{equation}

In this work, we adopt GIN~\cite{xu2019powerful} as backbone encoder, which is 1-WL~\cite{weisfeiler1968reduction} expressive. For more expressive GNN architectures, such as subgraph-based GNNs~\cite{zhao2021stars,zhang2021nested,you2021identity,maron2020learning} and $K$-hop message-passing GNNs~\cite{feng2022powerful,yao2023improving,nikolentzos2020k} as encoders, the OOD generalizability may further enhances as the expressive power increases.

\noindent{\bf Central Moment Discrepancy.} Central Moment Discrepancy (CMD) is a probalistic metric for domain-invariant representation learning, particularly used in the context of domain adaptation with neural networks. The main goal of CMD is to minimize the discrepancy between domain-specific latent feature representations directly in the hidden activation space. Unlike traditional methods such as Maximum Mean Discrepancy (MMD), CMD explicitly aligns higher-order moments of probability distributions in an order-wise manner. The CMD method utilizes central moments to quantify the differences between probability distributions. Let $P$ and $Q$ be two probability distributions with their respective central moments $\mu_k(P)$ and $\mu_k(Q)$ for order $k$. The CMD between these distributions can be defined as: $\operatorname{CMD}(P, Q)=\sum_{k=1}^K\left|\mu_k(P)-\mu_k(Q)\right|$, where $K$ is the highest order of moments considered. The central moments $\mu_k$ are defined as: $\mu_k(P)=\mathbb{E}_{x \sim P}\left[\left(x-\mathbb{E}_{x \sim P}[x]\right)^k\right]$.

\noindent{\bf OOD Generalization.} In the field of machine learning, especially in OOD settings, deep neural networks are known to exploit spurious features, leading to failures in generalization. Recently, there is an emerging line of work proposed to address this challenge. IRM~\cite{arjovsky2020invariant} aims to learn an optimal invariant association across diverse environment segments within the training data. Extending this concept, some studies have integrated multi-objective optimization with game theory~\cite{ahuja2020invariant,chen2022learning} or adopted invariant representation learning through adversarial training using deep neural networks~\cite{chang2020invariant,xu2021learning}; In parallel, Distributionally Robust Optimization (DRO)~\cite{BenTal2011RobustSO,lee2018minimax,gao2020wasserstein,duchi2020learning} methods have been developed to enhance OOD generalization. These methods focus on training models to perform robustly against the worst-case loss among diverse data groups; However, both IRM and DRO typically necessitate explicit environment partitions within the dataset, a requirement that is often impractical in real-world scenarios. This limitation has motivated research into invariance learning where such explicit environment partitions are not available. Environment Inference for Invariant Learning (EIIL)~\cite{creager2021environment} employs a two-stage training process involving biased model environment inference and subsequent invariance learning. Similarly, Heterogeneous Risk Minimization (HRM)~\cite{liu2021heterogeneous} addresses this issue by simultaneously learning latent heterogeneity and invariance. However, most of these methods focus on Euclidean data, and cannot be trivially adapted to graph-structured data. 

\noindent{\bf Graph Invariance Learning.} In recent years, there has been an increasing focus on learning graph-level representations that are robust to distribution shifts, particularly from the perspective of invariant learning. This growing interest has led to the development of two lines of research in graph invariance learning algorithms. The first line of research involves environment inference~\cite{yang2022learning,li2022learning} or environment augmentation~\cite{wu2022discovering,Liu_2022} algorithms, which infer environmental labels, or perform environment augmentation, and then use this information to learn graph invariant features. Another line of work do not explicitly address the issue of unobserved environment labels. Instead, these methods adopt alternative strategies to achieve invariant learning. For instance, CIGA~\cite{chen2022learning} utilizes contrastive learning within the same class labels, with the underlying assumption that samples with the same label share invariant substructures; DisC~\cite{fan2022debiasing}, conversely, leverages biased information to initially learn a biased graph, subsequently focusing on unbiased graphs for learning invariant features. However, these methods often rely on strong assumptions about the joint distribution $\mathbb{P}(S,Y)$, which can lead to potential failures in real-world scenarios. For example, DisC assume a strong correlation between $S$ and $Y$ to identify biased graph. While CIGA's ability to identify invariant subgraphs is contingent on a stronger correlation between $C$ and $Y$. Similarly, environment inference and augmentation algorithms typically assume a weak correlation between $S$ and $Y$; otherwise, $C$ might be erroneously included in $\widehat{\mathbf{h}}_s$, leading to the failure of environment inference or augmentation.

In this work, we adopt a self-supervised learning approach based on the infomax principle to extract spurious features with provable guarantee. This method alleviates the dependence on the label $Y$, and reduces the reliance on the correlation between $S$ and $Y$, establishing a new inductive bias that remains robust under varying degrees of correlation in $\mathbb{P}(S,Y)$. 

\noindent{\bf Identifiability in Self-Supervised Learning.} Self-supervised learning with augmentations has gained huge success in learning useful graph representations~\citep{dgi,graphcl}.
Existing analysis of self-supervised learning focuses on showing the desired property such as identifying the content from style~\citep{ssl_sep}, or invariant subgraph from spurious one~\citep{chen2022learning,chen2023does,mole_identify}. In contrast, we show that infomax principle tends to learn the spurious features under suitable conditions. 

\section{Proofs for Theorems and Propositions}

\subsection{Proof for Theorem~\ref{thm:spu_infomax}}\label{proof:spu_infomax}
\begin{theorem}\label{thm:spu_infomax_appdx}[Restatement of Theorem~\ref{thm:spu_infomax}]
    Given the same data generation process as in Fig.~\ref{fig:scm} with $H(S)=H(C)=\delta_f$, assuming the node representations encode proper information of the underlying latent factors (i.e., $\delta_r\geq I(\widehat{\vh}_{i};C)- I(\widehat{\vh}_{i};S)\geq \delta_l, \forall i\in G_c$ and $\delta_r\geq I(\widehat{\vh}_{i};S)- I(\widehat{\vh}_{i};C)\geq \delta_l, \forall i\in G_s$),  the graph representation $\widehat{\vh}_G \in \mathbb{R}^k$ have sufficient capacity to encode $k$ independent features $\{\widehat{\vh}_G[j]\}_{j=1}^k$ with $H(\widehat{\vh}_G[j])\leq \delta_f$, then, if $|G_s|/|G_c|> \delta_r/\delta_l$, then the graph representation elicited by the infomax principle~(Eqn.~\ref{infomax_eq}) exclusively contain spurious features $S$, i.e.,

    $$
    \mathbf{h}_s=\argmax _\theta \frac{1}{|G||\mathcal{G}|} \sum_{G \in \mathcal{G}} \sum_{i \in|G|} I\left(\widehat{\mathbf{h}}_i; \widehat{\mathbf{h}}_G\right).
    $$
\end{theorem}
\begin{proof}
    Our proof is established by contradiction. To begin with, without loss of generality, let us assume that in the final solution $\widehat{\vh}_G$, there exist some features $i_c\in I_C\subseteq\{1,...,k\}$ in $\widehat{\vh}_G$ exclusively encoding information of $C$, such that,
    \[
        I(\widehat{\vh}_{i_c};C)=H(C), 
    \]
    where we denote $i_c$-th feature of $\widehat{\vh}_G$ as $\widehat{\vh}_{i_c}$ for the clarity of notation. Correspondingly, we could also find the complementary dimensions as $i_s\in I_S$ such that
    \[
    I(\widehat{\vh}_{i_s};S)=H(S). 
    \]
    
    Consider a single graph $G$, to solve $\max_\theta \; I\big(\widehat{\mathbf{h}}_i, \widehat{\mathbf{h}}_G\big)$ (Eqn.~\ref{infomax_eq}), we can expand $\widehat{\vh}_G$ and write out the mutual information with respect to each node as follows: 
    
    \begin{equation}
    \begin{aligned}
        I_\text{DI} &= I(\widehat{\vh}_i;\cup_{i_c\in I_C}\widehat{\vh}_{i_c},\cup_{i_s\in I_S}\widehat{\vh}_{i_s})\\
        &=\sum_{i\in G}(\sum_{i_c\in I_C}I(\widehat{\vh}_i;\widehat{\vh}_{i_c})+\sum_{i_s\in I_S}I(\widehat{\vh}_i;\widehat{\vh}_{i_s}))\\
        &=\sum_{i\in G_c}(\sum_{i_c\in I_C}I(\widehat{\vh}_i;\widehat{\vh}_{i_c})+\sum_{i_s\in I_S}I(\widehat{\vh}_i;\widehat{\vh}_{i_s}))+\sum_{i\in G_s}(\sum_{i_c\in I_C}I(\widehat{\vh}_i;\widehat{\vh}_{i_c})+\sum_{i_s\in I_S}I(\widehat{\vh}_i;\widehat{\vh}_{i_s}))\\
        &=\sum_{i\in G_c}(\sum_{i_c\in I_C}I(\widehat{\vh}_i;C)+\sum_{i_s\in I_S}I(\widehat{\vh}_i;S))+\sum_{i\in G_s}(\sum_{i_c\in I_C}I(\widehat{\vh}_i;C)+\sum_{i_s\in I_S}I(\widehat{\vh}_i;S))\\
    \end{aligned}
    \end{equation}
    Then, considering switching an index $i_c'\in C$ to encode $S$, we will have the information changes to $I_\text{DI}$ as the following
    \begin{equation}
    \begin{aligned}
        \Delta I_\text{DI} 
        &=\sum_{i\in G_c}(-I(\widehat{\vh}_i;C)+I(\widehat{\vh}_i;S))+\sum_{i\in G_s}(-I(\widehat{\vh}_i;C)+I(\widehat{\vh}_i;S))\\
        &=-\sum_{i\in G_c}(I(\widehat{\vh}_i;C)-I(\widehat{\vh}_i;S))+\sum_{i\in G_s}(I(\widehat{\vh}_i;S)-I(\widehat{\vh}_i;C)),\\
    \end{aligned}
    \end{equation}
    where we can upper bound the first item 
    \[
    \sum_{i\in G_c}(I(\widehat{\vh}_i;C)-I(\widehat{\vh}_i;S))\leq |G_c|(I(\widehat{\vh}_i;C)-I(\widehat{\vh}_i;S))=|G_c|\delta_r,
    \]
    and lower bound the second item
    \[
    \sum_{i\in G_s}(I(\widehat{\vh}_i;S)-I(\widehat{\vh}_i;C))\geq |G_s|(I(\widehat{\vh}_i;S)-I(\widehat{\vh}_i;C))=|G_s|\delta_l.
    \]
    Then, it suffices to know that 
    \[
    \Delta I_\text{DI}\geq |G_s|\delta_l-|G_c|\delta_r>0,
    \]
    and hence switching any node $i_c\in I_C$ to one for $S$ increases the infomax objective. Therefore, we conclude that $\vh_s$ maximizes the objective of the infomax principle. We conclude the proof for Theorem~\ref{thm:spu_infomax}.
\end{proof}

\subsection{Proof for Theorem \ref{theorem:inv_eq}}
\label{proof:theo_inv_eq}
To prove Theorem \ref{theorem:inv_eq}, we show both necessity and sufficiency that Eqn. \ref{inv_eq1}, serving as the penalty term for $\mathcal{L}_{GT}$, will encourage the decorrelation of $\widehat{\mathbf{h}}_c$ and $\mathbf{h}_s$, and elicit invariant representations $\mathbf{h}_c$. First we propose a proposition to aid the proof. 

\begin{proposition}
\label{prop:bins_max_ent}

 Let $s_{i} \in \mathbb{R}_+$ and $s_{j} \in \mathbb{R}_+$ denote the empirical random variable for the logits values of sample that belongs to $B_i$ and $B_j$ respectively. For any two bins $B_i$ and $B_j$ with equal size samples (denoted by $K$), and located symmetrically around $0.5$, i.e., $|s_i|=|s_j|$, $d\left(\mathbb{E}\left[s_{i y}\right], 0.5\right)=d\left(\mathbb{E}\left[s_{jy}\right], 0.5\right)=\epsilon$, and where $sign(\mathbb{E}[s_{i}]-0.5)=sign(0.5-\mathbb{E} [s_{j}])$. The optimal and unique solution for $\hat{s}$ of the following cross-entropy function $l(\cdot)$ is $0.5$:

\begin{equation}
\begin{aligned}
l(\widehat{s}) & =K\left(-\mathbb{E}\left[s_i\right] \log \widehat{s}-\mathbb{E}\left[1-s_i\right] \log (1-\widehat{s})\right)+K\left(-\mathbb{E}\left[s_j\right] \log \widehat{s}-\mathbb{E}\left[1-s_j\right] \log (1-\widehat{s})\right) \\
& =K(-(0.5+\varepsilon) \log \widehat{s}-(1-0.5-\varepsilon) \log (1-\widehat{s}))+K(-(0.5-\varepsilon) \log \widehat{s}-(1-0.5+\varepsilon) \log (1-\widehat{s}))
\end{aligned}
\end{equation}

\end{proposition}

\begin{proof}
Simplifying $l(\widehat{s})$, we get:

$$
l(\widehat{s})=-K \log (\widehat{s})-K \log (1-\widehat{s}).
$$

Calculating $\frac{d}{d \hat{s}} l(\hat{s})$, we get:

\begin{equation}
\begin{aligned}
    \frac{d}{d \widehat{s}} f(\widehat{s}) & = \frac{d}{d \widehat{s}}(-K \log (\widehat{s})-K \log (1-\widehat{s})) \\
& = -K \cdot\left(\frac{1}{\hat{s}}-\frac{1}{1-\hat{s}}\right) \\
& = \frac{K(1-2 \widehat{s})}{\widehat{s}(1-\widehat{s})}.
\end{aligned}
\end{equation}

Setting $\frac{K(1-2 \widehat{s})}{\widehat{s}(1-\widehat{s})}=0$, we get $\hat{s}=0.5$. We conclude the proof.

\end{proof}

With Prop. \ref{prop:bins_max_ent}, we first prove the necessity, i.e., given $\widehat{\mathbf{h}}_c=\mathbf{h}_c$, Eqn. \ref{inv_eq1} satisfies $\max \; H\left(\mathbf{h}_s \mid \widehat{\mathbf{h}}_c\right)$.

\begin{proof}
Given $\widehat{\mathbf{h}}_c=\mathbf{h}_c$, and under the same class label $Y=y$, all samples that belong to class $y$ are encoded into $\mathbf{h}_c$, as $C$ is causally related to $Y$. Under the assumption of Theorem $\ref{theorem:inv_eq}$, we have $K$ pairs of symmetric bins for any label $Y$. Using Prop. \ref{prop:bins_max_ent}, we know that for any pair of symmetrical bins $B_i$ and $B_j$, the optimal solution for the estimated logits $\widehat{s_y}=\rho(\mathbf{h}_c) \in \mathbb{R}^+$ is $0.5$, which maximizes the prediction entropy. As the $K$ pairs of bins simultaneously achieve the same optimal solution, and $l(\cdot)$ is a convex function, $\widehat{s_y}=0.5$ is both optimal and unique. We conclude the proof.
\end{proof}

The necessity condition shows that $\widehat{\mathbf{h}}_c=\mathbf{h}_c$ is indeed one of the feasible solutions for Eqn. \ref{inv_eq1}, which decorrelates $\mathbf{h}_s$ and $\widehat{\mathbf{h}}_c$. We remain to show that $\widehat{\mathbf{h}}_c=\mathbf{h}_c$ is the unique solution for Eqn. \ref{inv_eq1}, when we optimize $\mathcal{L}=\mathcal{L}_{G T}+\lambda \mathcal{L}_{I n v}$, where Eqn. \ref{inv_eq1} serves as a penalty term $\mathcal{L}_{Inv}$. To demonstrate the sufficiency, we will divide our analysis into three distinct scenarios. In each scenario, we will compare and discuss the empirical risks associated with $\mathcal{L}_{GT}$ and $\mathcal{L}_{Inv}$.

\begin{proof}
The empirical risk for $l(\cdot)$ is defined as: $\hat{R}_n(l) \stackrel{\text { def }}{=} \frac{1}{n} \sum_{i=1}^n \mathbb{I}\left(l\left(X_i\right) \neq Y_i\right)$. We consider three scenarios: 1) The solution only contains $C$, i.e., $\widehat{\mathbf{h}}_c=\mathbf{h}_c$, which we have proved that it is one of the feasible solution. 2) The solution only contains $S$, i.e., $\widehat{\mathbf{h}}_c=\mathbf{h}_s$. 3) The solution contains a mixing ratio of $C$ and $S$. We need to show that for case 1), the empirical risk for $\mathcal{L}_{G T}+\lambda \mathcal{L}_{Inv}$ is the smallest. We consider the first two cases. For any pair of symmetrical bins with equal sample size $n$ ($2n$ in total):

\begin{itemize}
    \item For case 1, $\hat{R}_n(\mathcal{L}_{GT})$ would be $0$ as  $\widehat{\mathbf{h}}_c=\mathbf{h}_c$. $\hat{R}_n(\mathcal{L}_{Inv})=n$ as in this case the estimated probability $p=0.5$, according to Prop. \ref{prop:bins_max_ent}.
    \item For case 2, $\hat{R}_n(\mathcal{L}_{GT})=n$ and $\hat{R}_n(\mathcal{L}_{Inv})=0$ since we can encode different $\widehat{\mathbf{h}}_c$ for different samples to fit the spurious patterns. However, only half of them are correlated to $Y$.
\end{itemize}

From the above discussion, we conclude that when $\lambda<1.0$, the solution for $\mathcal{L}$ would prefer $\widehat{\mathbf{h}}_c=\mathbf{h}_c$, whereas when $\lambda>1.0$, the solution would prefer $\widehat{\mathbf{h}}_c=\mathbf{h}_s$. This is consistent with our experiment results in Section \ref{app:exp}: when $\lambda>1.0$, the model performance degrade dramatically.

Finally, we also need to rule out case 3. For simplicity, we assume that the encoder $h(\cdot) \sim \operatorname{Bernoulli}(r)$ is stochastic which follows a Bernoulli distribution with parameter $r$. Then in each bin, there are $r$ samples that are encoded using $C$ and $1-r$ encoded using $S$, then with a similar technique, we can derive that: $\hat{R}_n(\mathcal{L}_{GT})=(1-r)n$, and $\hat{R}_n(\mathcal{L}_{Inv})=rn$, which can be obtained as a convex combination of case 1 and case 2 for $\mathcal{L}_{GT}$ and $\mathcal{L}_{Inv}$ respectively. 
As for case 3, $\hat{R}_n(\mathcal{L}_{GT})=(1-r)n>0$ when $r<1$, for any value $r$, we can find a suitable $\lambda$ that make the empirical risk for case 1 to be smallest among the 3 scenarios. In other words, $\widehat{\mathbf{h}}_c=\mathbf{h}_c$ would be optimal and unique solution when $\mathcal{L}=\mathcal{L}_{G T}+\lambda \mathcal{L}_{Inv}$. We conclude the proof for Theorem \ref{theorem:inv_eq}.

\end{proof}

\section{More Discussions on \texttt{EQuAD} Framework}
\label{app:discuss_equad_framework}
\oursfull (\ours) is a flexible learning framework, where off-the-shelf algorithms can serve as plug-ins for specific implementations for each step. In the first two steps, \textit{Encoding} and \textit{Quantifying}, we adhere to infomax-based self-supervised learning and model-based quantification due to their effectiveness. For the third step, \textit{Decorrelation}, multiple options are available. We note that the learning objective in Eqn.~\ref{inv_eq2} demonstrates superior performance across multiple datasets. However, in cases of significant bias, this objective may encounter optimization challenges, despite its theoretical soundness. This issue often arises when $S$ and $Y$ are highly correlated, leading minibatch gradient descent to inadvertently reinforce gradients towards spurious patterns, as samples highly correlated with $Y$ predominate in each minibatch. A potential solution involves full-batch gradient descent with sample reweighting, but this approach is memory-intensive and impractical for large datasets. To overcome this limitation, we propose an alternative learning objective for decorrelation, drawing inspiration from DisC~\cite{fan2022debiasing}. This approach utilizes ERM and assigns greater weight to sample $i$ from class $Y=y$ with logits value $\widehat{s}_{iy}$ close to zero, indicating a lack of spurious patterns. By employing a reweighted cross-entropy loss, this method focuses on learning invariant features. The formulation of the alternative loss function $\mathcal{L}_{GT}$ is outlined as follows:

\begin{equation}
\label{erm_eq_alter}
\mathcal{L}_{GT}=w_i CE \big(y_i, \widehat{y_i}\big), \forall i \in [N], \text { s.t. } w_i=\frac{1-s_{iy}^{\gamma}}{\gamma},
\end{equation}

where $CE(\cdot,\cdot)$ is the cross-entropy loss. In Eqn.~\ref{erm_eq_alter}, each sample $i$ with ground-truth label $y$ get reweighted such that the model training can focus on samples exhibit weaker spurious correlation, thus implicitly achieve the optimality of Eqn.~\ref{max_opt_CS}. While DisC~\cite{fan2022debiasing} requires learning the biased graphs at first using generalized cross entropy (GCE)~\cite{zhang2018generalized} loss, EQuAD can directly obtain the spurious (biased) information from step 1 and 2, without presumptive assumptions on $\mathbb{P}(S,Y)$. We adopt Eqn~\ref{erm_eq_alter} for two-piece graph datasets without the need of $\mathcal{L}_{Inv}$, and observe better predictive performance and faster convergence speed. 

\section{More Details on Experiments about Representation Quality Analysis}
\label{app:rep_quality}

Recent studies have indicated that for Euclidean data, even when neural networks heavily rely on spurious features and perform poorly on minority groups where the spurious correlation is broken, they still learn the invariant, or core features sufficiently well \cite{kirichenko2023layer}. In our work, we conducted a similar experiment with the SPMotif datasets to investigate whether ERM has already learned effective invariant features. Specifically, we first trained an encoder $h(\cdot)$ on the training set using ERM, then froze this encoder and obtain representations for each sample. Subsequently, we added a 2-layer MLP on top of $h(\cdot)$ and re-trained this MLP for feature reweighting on half of the validation set, where the spurious correlation does not hold. If ERM truly learned causally-related features, then the feature reweighting via the validation set should be able to exploit the invariant features to achieve OOD generalization. The experiment results are illustrated in left part of Figure \ref{fig:emb_quality}. As we can see, for both SPMotif-0.60 and SPMotif-0.90, the test accuracy is already on par with or even better than many state-of-the-art invariance learning methods. This suggests that \textit{ERM has already been able to learn effective invariant features}, which is contrary to our objective. In contrast, when utilizing the representations learned from infomax-based SSL, the accuracy after feature reweighting exhibits a significant gap compared with ERM, indicating that the representations derived from infomax-based SSL predominantly contains spurious features.

\section{Algorithmic Pseudocode}

We provide the following pseudocode to facilitate a better understanding of our algorithm and framework. The code will be open-sourced upon acceptance of our work.

\label{appendix:pseudo-code}

\begin{algorithm}[H]
\caption{The \ours Framework}
\KwIn{graph dataset $\mathcal{G}$, set of epochs $\mathcal{P}$, number of classifiers $M$, total number of training epochs $P$ for decorrelation}
\KwOut{Model $f = \rho \circ h$, composed of encoder $h(\cdot)$ and classifier $\rho(\cdot)$}

\textbf{Step 1: Encoding}\;
Initialize encoder $h(\cdot)$\;
\For{each model architecture $i \in [K]$}{
    \For{each epoch $j \in \mathcal{P}$}{
        Train $h(\cdot)$ using Eqn.~\ref{infomax_eq}\;
        Generate latent representations $\mathbf{H}^{(i, j)}$\;
    }
}
Collect all representations $\mathcal{H} := \{\mathbf{H}^{(i, j)}\}$\;

\textbf{Step 2: Quantifying}\;
Initialize classifiers $g(\cdot)$\;
\For{each representation $\mathbf{H}^{(i, j)} \in \mathcal{H}$}{
    Train classifier $g(\cdot)$ using ERM, taking $\mathbf{H}^{(i, j)}$ and $Y$ as inputs\;
    Compute logits matrix $\mathbf{S}^{(i, j)} = g(\mathbf{H}^{(i, j)}, Y)$\;
}
Collect all logits matrices $\mathcal{S} := \{\mathbf{S}^{(i, j)}\}$\;
Select top-$M$ classifiers $g(\cdot)$ according to the lowest validation metric\;

\textbf{Step 3: Decorrelation}\;
Reinitialize model $f(\cdot)$\;
\For{epoch $e \in \{1, 2, \ldots, P\}$}{
    \For{each minibatch $B$ in $\mathcal{G}$}{
        Compute logits from top-$M$ logits matrices in $\mathcal{S}$\;
        Minimize Eqn. \ref{final_loss} to update $\rho(\cdot)$ and $h(\cdot)$\;
    }
}

\end{algorithm}

\section{Complexity Analysis}

\noindent\textbf{Time Complexity.} The time complexity of the \texttt{EQuAD} framework depends on the specific GNN encoder employed. In this work, we utilize GIN~\cite{xu2019powerful}, which is a 1-WL GNN. Consequently, the time complexity is $\mathcal{O}(CkmF)$, where $k$ is the number of GNN layers, $F$ is the feature dimensions and $m$ denotes the number of edges in a graph $G$. \texttt{EQuAD} incurs an additional constant $C>1$ due to the multiple-stage learning paradigm.

\noindent\textbf{Space Complexity.} EQuAD also incurs additional memory overhead as in the encoding step, it needs to store multiple matrices, and in the quantifying stage the matrices are transformed into logits matrices (which is much lower dimensional). In the implementation, we save the graph representation matrices (and corresponding labels) obtained from encoding stage in the disk, and load them one by one, followed by transforming them into logits matrices, hence the memory complexity is $\mathcal{O}(H|\mathcal{C}||\mathcal{B}|+|\mathcal{B}|m)$, here $|\mathcal{C}|$ is the number of class labels, $H$ is the number of the graph representation matrices selected from the quantifying step, and $|\mathcal{B}|$ denotes the batch size. As $H$, $|\mathcal{C}|$ and $|\mathcal{B}|$ are usually small integers, the memory cost is affordable.

We provide a empirical running time analysis in Table~\ref{app:running_time}. As illustrated, the running time of EQuAD is comparable to or less than that of most invariance learning methods, also exhibiting lower variance, since we do not adopt early stop and only need to run for 50 epochs for all datasets. Specifically, the first step accounts for around 65\% of the total time, while steps 2 and 3 together constitute the remaining 35\%. 

\begin{table}[]
\centering
\caption{Running time analysis in seconds for various baseline methods and \texttt{EQuAD} on EC50 datasets.}
\label{app:running_time}
\begin{tabular}{@{}cccc@{}}
\toprule
\textbf{Method} & \textbf{EC50-Sca} & \textbf{EC50-Assay} & \textbf{EC50-Size} \\ \midrule
ERM             & 113.97±1.56       & 169.83±21.76        & 224.54±63.38       \\ \midrule
IRM             & 1102.51±17.31     & 1719.48±215.50      & 1035.22±30.71      \\
V-Rex           & 932.81±1.16       & 1498±137.18         & 886.07±0.27        \\
CIGA            & 2179.94±540.88    & 2676.04±897.95      & 1822.99±551.09     \\
GREA            & 1812.90±273.94    & 2461.01±876.19      & 1661.67±375.18     \\
GALA            & 1699.81±235.46    & 2176.58±638.76      & 1365.27±240.32     \\
GSAT            & 1140.90±285.10    & 1609.15±322.02      & 1135.12±326.98     \\ \midrule
EQuAD           & 878.34±12.31      & 1240.88±36.32       & 1326.61±21.79      \\ \bottomrule
\end{tabular}
\end{table}

\section{Experimental Details}

\subsection{Datasets}
\label{app:datasets}

We provide a more detailed introduction of the datasets adopted in the experiments as follows.

\textbf{SPMotif datasets.} Following \cite{wu2022discovering}, we generate a 3-class synthetic datasets based on BAMotif \cite{ying2019gnnexplainer}. In these datasets, each graph comprises a combination of invariant and spurious subgraphs, denoted by $I$ and $S$. The spurious subgraphs include three structures (Tree, Ladder, and Wheel), while the invariant subgraphs consist of Cycle, House, and Crane. The task for a model is to determine which one of the three motifs (Cycle, House, and Crane) is present in a graph. A controllable distribution shift can be achieved via a pre-defined parameter $b$. This parameter manipulates the spurious correlation between the spurious subgraph $G_s$ and the ground-truth label $Y$, which depends solely on the invariant subgraph $G_c$. Specifically, given the predefined bias $b$, the probability of a specific motif (e.g., House) and a specific base graph (Tree) will co-occur is $b$ while for the others is $(1-b) / 2$ (e.g., House-Ladder, House-Wheel). When $b=\frac{1}{3}$, the invariant subgraph is equally correlated to the three spurious subgraphs in the dataset. In SPMotif datasets, $S$ is directly influenced by $C$, and $C$ is causally related with $Y$, thus satisfies our data generating assumption as FIIF.

\textbf{Two-piece graph datasets.} To validate the effectiveness of EQuAD under the PIIF data generating process, we adopt the two-piece graph datasets\cite{chen2023does}. These datasets employ parameters $\alpha^e$ and $\beta^e$ to control the correlations between $C$ and $Y$, and between $S$ and $Y$ respectively in different environments. A formal definition of the two-piece graph is presented as follows:

\begin{definition}[Two-piece graphs]
Each environment $E=e$ is defined with two parameters, $\alpha_e, \beta_e \in$ $[0,1]$, and the dataset $\left(G^e, Y^e\right) \in \mathcal{D}_e$ is generated as follows:
\begin{enumerate}
    \item[(a)] Sample $Y^e \in\{-1,1\}$ uniformly;
    \item[(b)] Generate $G_c$ and $G_s$ via: $G_c:=f_{\text{gen}}^{G_c}\left(Y^e \cdot \operatorname{Rad}\left(\alpha_e\right)\right), G_s:=f_{\text{gen}}^{G_s}\left(Y^e \cdot \operatorname{Rad}\left(\beta_e\right)\right)$, where $f_{\text{gen}}^{G_c}, f_{\text{gen}}^{G_s}$ map the input $\{-1,1\}$ to a corresponding graph selected from a given set, and $\operatorname{Rad}(\alpha)$ is a random variable taking value -1 with probability $\alpha$ and +1 with $1-\alpha$;
    \item[(c)] Synthesize $G^e$ by randomly assembling $G_c$ and $G_s: G^e:=f_{g e n}^G\left(G_c, G_s\right)$.
\end{enumerate}
\end{definition}

Specifically, we adopt BAMotif \cite{ying2019gnnexplainer} to generate 3 variants of 3-class two-piece graph datasets, with different correlation degrees of $H(S|Y)$ parametrized by $(\alpha^e,\beta^e)$, where we can examine both scenarios where $H(C|Y)>H(S|Y)$ and $H(C|Y)<H(S|Y)$, to verify the effectiveness of the baseline methods and our approach.

\textbf{DrugOOD.} The DrugOOD benchmark\cite{ji2022drugood} is specifically designed for OOD challenges in AI-aided drug discovery. It comprises a diverse collection of datasets focusing on the prediction of drug-target interactions. Each dataset within DrugOOD is derived from real-world bioactivity data, encompassing a range of drug compounds and their target proteins. The datasets are categorized based on two bioactivity measures: IC50 and EC50, representing half-maximal inhibitory concentration and half-maximal effective concentration, respectively. DrugOOD datasets are split into subsets using three distinct criteria: Assay, Scaffold, and Size. The Assay-based split groups data according to biological assays, reflecting variations in experimental conditions. The Scaffold-based split focuses on the chemical structure, categorizing compounds by their core molecular scaffolding. The Size-based split, on the other hand, divides the data based on the size of the molecular compounds, offering insights into size-dependent drug properties. With the two measurements (IC50 and EC50), and three environment-splitting strategies (assay, scaffold and size), we obtain 6 datasets, and each dataset contains a binary classification task for drug target binding affinity prediction.
    
\textbf{Open Graph Benchmark(OGB).}  We also consider two molecule datasets, MOLBACE and MOLBBBP, from Open Graph Benchmark \cite{hu2021open}. We use the scaffold splitting procedure as OGB adopted, where different molecules are structurally separated into different subsets, which provides a more realistic estimate of model performance in experiments \cite{C7SC02664A}. 

\subsection{Experiment Settings}

\textbf{Encoding.} For encoding stage, we adopt MVGRL\cite{hassani2020contrastive} to obtain the latent representations. MVGRL \cite{hassani2020contrastive} utilize global-local MI maximization as learning objective, In addition, it utilizes Personalized PageRank\cite{Page1999ThePC} and heat kernel\cite{Kondor2002DiffusionKO} as data augmentation methods to generate correlated structural views for contrastive learning. Intuitively, the structural augmentation corrupts the invariant substructure by randomly adding edges to the neighboring nodes, thus further facilitates the learning of spurious patterns under OOD settings. Across all the datasets, we set  teleport probability $\alpha=0.2$, and  diffusion time $t=5$. To obtain a collection of latent representation $\mathcal{H}$, we set pre-defined training epochs $\mathcal{P}=\left\{50,100,150\right\}$, number of layers in GNN encoder $L=\left\{2,3,5\right\}$, hidden dimensions $H=\left\{32,64\right\}$, leading to 18 embedding matrices for each dataset. 

\textbf{Quantifying.} To quantify the latent representations in $\mathcal{H}$ and obtain $\mathcal{S}$, we adopt linear svm\cite{hearst1998support} for its computational efficiency, followed by probability calibration \cite{platt1999probabilistic,zadrozny2001obtaining} to get logits values. For all datasets, the regularization parameters $C$ of linear svms are searched over $\left\{10,1000\right\}$. We set cross validation to be $5$ for probability calibration classifier.

\textbf{Decorrelation.} For the decorrelation step, we utilize Eqn. \ref{final_loss} and search over the following hyperparameters:

\begin{itemize}
    \item $\gamma:\left\{0.7,0.5,0.3,0.1,0.05\right\}$
    \item  $\tau:\left\{1.0,0.5,0.25,0.1,0.01\right\}$
    \item  $\lambda:\left\{1e-3,1e-2,1e-1\right\}$
\end{itemize}

As discussed in Section~\ref{app:discuss_equad_framework}, when the dataset exhibits high bias, employing Eqn. \ref{inv_eq2} can lead to gradient conflicts between $\mathcal{L}_{GT}$ and $\mathcal{L}_{Inv}$, posing challenges in optimization. A potential solution is to adopt full-batch training; however, this approach is not memory efficient and is impractical for large-scale datasets. To address this issue, we propose an alternative loss that encourages decorrelation between $\widehat{\mathbf{h}}_c$ and $\mathbf{h}_s$ which is efficient and effective when the bias is high, as outlined in Eqn.~\ref{erm_eq_alter}. For two-piece graph datasets, we adopt Eqn.~\ref{erm_eq_alter} for model training, which simplifies the optimization process, and also enhances performance. The hyperparameter $q$ is searched over: $\left\{0.9,0.7,0.5,0.3,0.1,0.01\right\}$. 

\textbf{Backbone encoder.} We adpot GIN\cite{xu2019powerful} as backbone encoder for all the datasets, and the number of layers $L$ are searched over $\left\{3,5\right\}$ across all the datasets for EQuAD and all other baseline methods. For synthetic datasets, the hidden dimension is uniformly set to 32 across all methods, while for the DrugOOD and OGBG datasets, it is set to 128. 

\textbf{Other experiment setting details.} For all the methods and datasets, the experiments are ran for three times, with random seeds $\small\{1,2,3\small\}$. The batch size is 32 for all experiments, with a learning rate of $1e-3$ is applied across all experiments. We adopt Adam optimizer~\cite{kingma2017adam} for model training. For \texttt{EQuAD}, the experiments for SPMotif and DrugOOD datasets are ran for 50 epochs, and for OGBG datasets, the experiments are ran for 100 epochs. We do not use early stopping and learning rate scheduler for all the experiments. 

\subsection{Implementation Details}

We implement our method and all baseline methods using PyTorch\cite{paszke2019pytorch} and PyTorch Geometric\cite{fey2019fast}. To generate latent embeddings with MVGRL\cite{hassani2020contrastive}, we adopt PyGCL\cite{Zhu:2021tu} package for the implementation. We utilize linearSVC and CalibratedClassifierCV in scikit-learn\cite{scikit-learn} for the implementation of linear SVM and probability calibration classifier respectively. We ran all our experiments on Linux Servers with GeForce RTX 4090 with CUDA 11.8.

\begin{figure}[h]
\centering
\vspace{-1pt}
\scalebox{0.95}{
\includegraphics[width=0.48\textwidth]{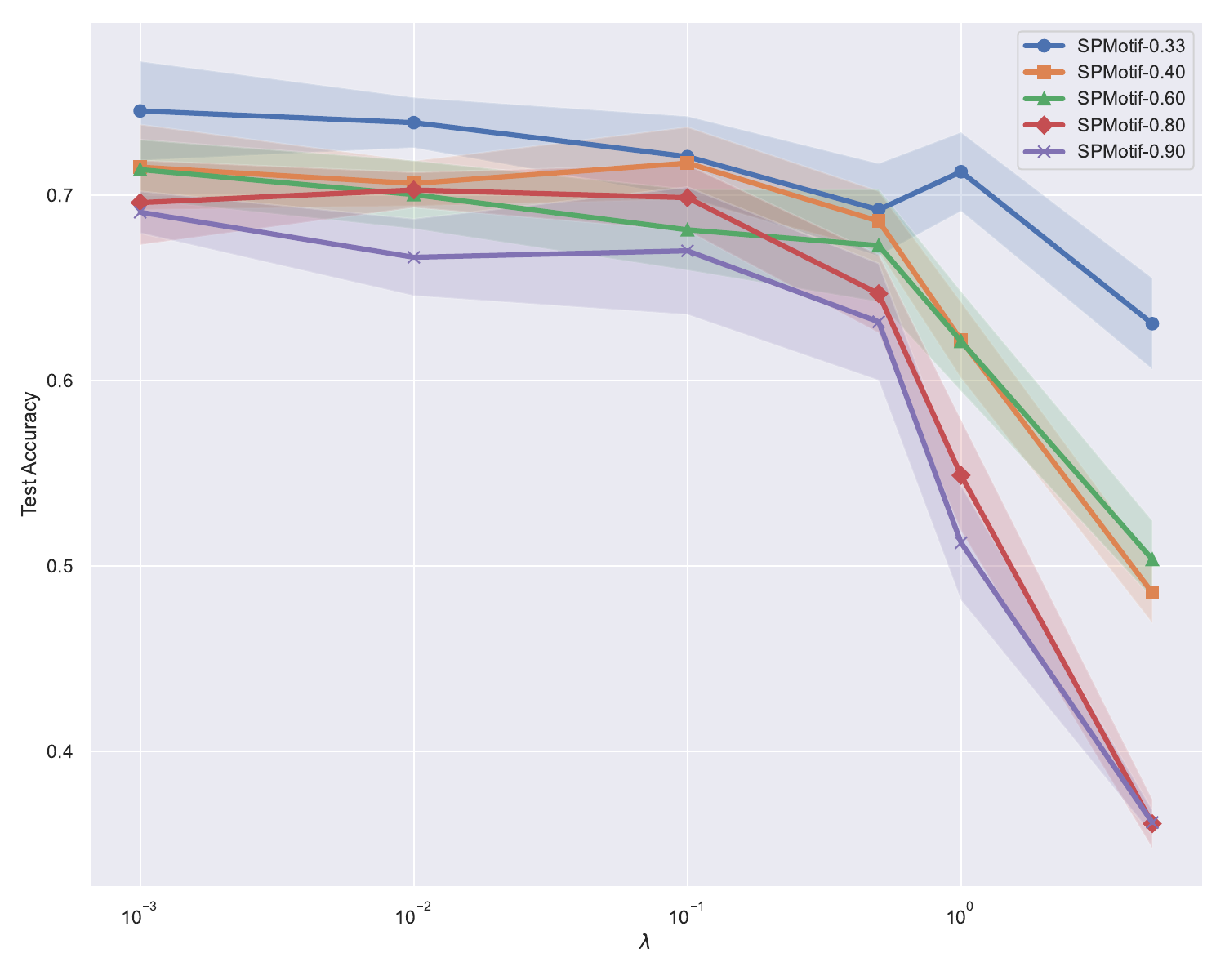}}
\vspace{-10pt}
\caption{Ablation study on $\lambda$. Notably, when $\lambda$ is greater than 1, the performance degrades dramatically, which is consistent with our theoretical results in Seciton~\ref{proof:theo_inv_eq}.} 
\label{fig:sl_abstudy}
\end{figure}

\begin{figure*}[t]
\centering
\vspace{-1pt}
\scalebox{0.95}{
\includegraphics[width=0.95\textwidth]{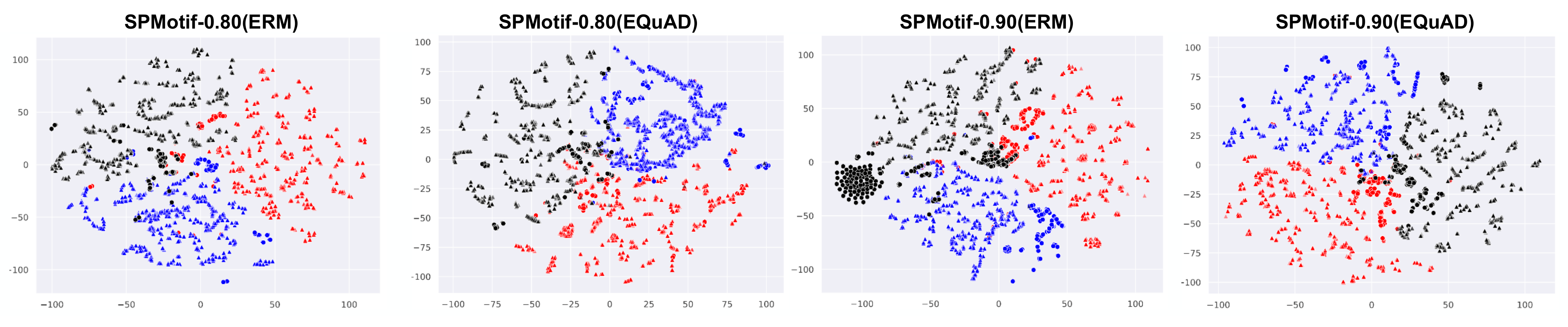}}
\vspace{-8pt}
\caption{Latent representations of data samples in 2d space for SPMotif datasets, where samples from different labels are represented by different colors,  within the same class label, we use \textit{cycles} and \textit{triangles} to represent data samples with high and low spurious correlations to $Y$. Representations obtained from EQuAD exhibit higher overlapping for the \textit{cycles} and \textit{triangles} points within the same classes, thus leading to better graph invariant features.} 
\label{fig:tsne_spm89}
\end{figure*}

\begin{figure}[t]
\centering
\vspace{-1pt}
\scalebox{0.98}{
\includegraphics[width=0.95\textwidth]{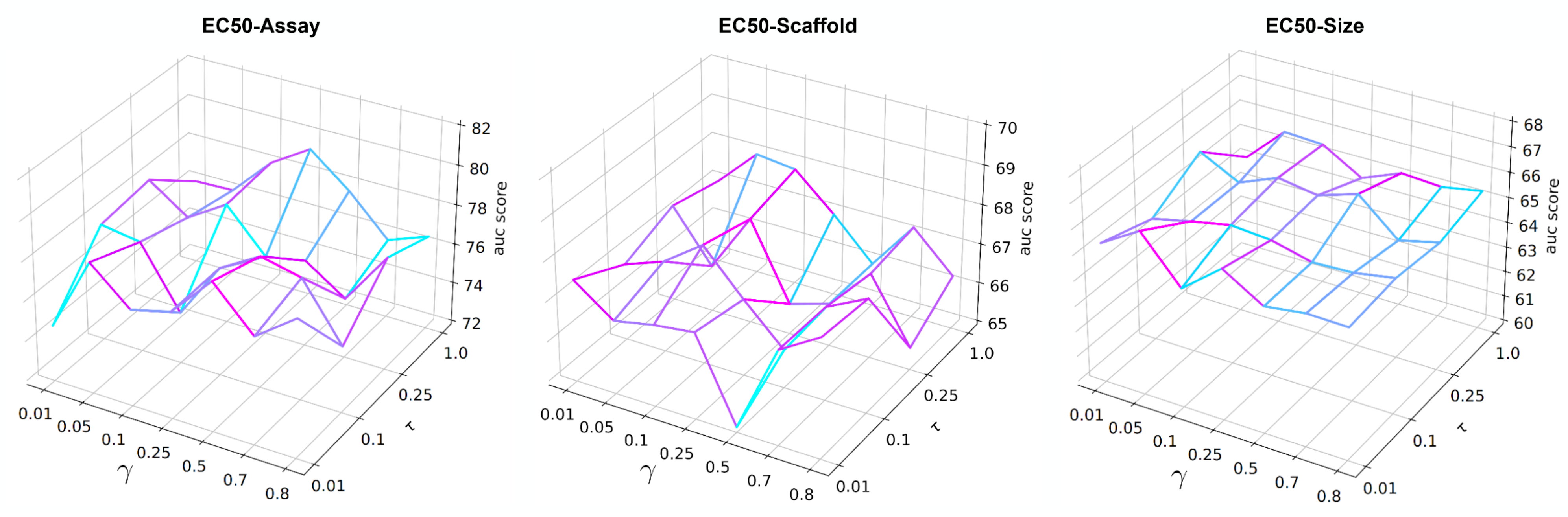}}
\vspace{-10pt}
\caption{Sensivitity on hyperparameters of $\gamma$ and $\tau$.} 
\label{fig:ec50_hyper_3d}
\end{figure}

\begin{figure}[t]
\centering
\vspace{-1pt}
\scalebox{0.98}{
\includegraphics[width=0.95\textwidth]{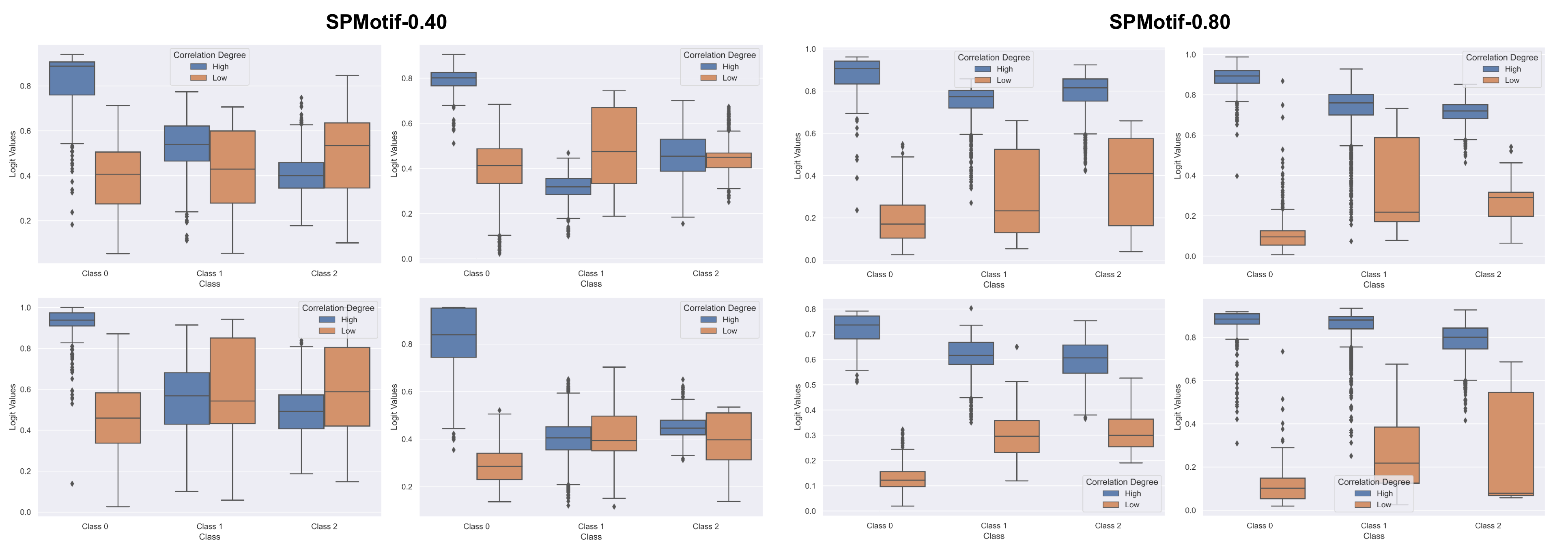}}
\vspace{-10pt}
\caption{Logits distributions for all three classes of SPMotif-0.40 and SPMotif-0.80.} 
\label{fig:boxplots}
\end{figure}

\begin{figure}[t]
    \centering
    \includegraphics[width=\linewidth]{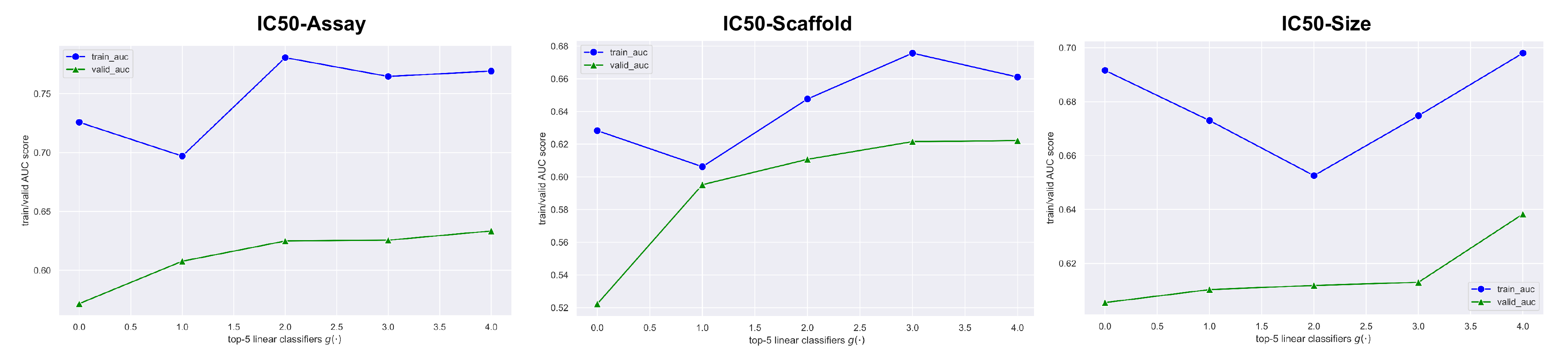}
    \caption{Train and validation AUC score for top-5 linear classifiers $g(\cdot)$ in step 2 for IC50 datasets.}
    \label{fig:ic50_svm}

    \vspace{0.5cm} 

    \includegraphics[width=\linewidth]{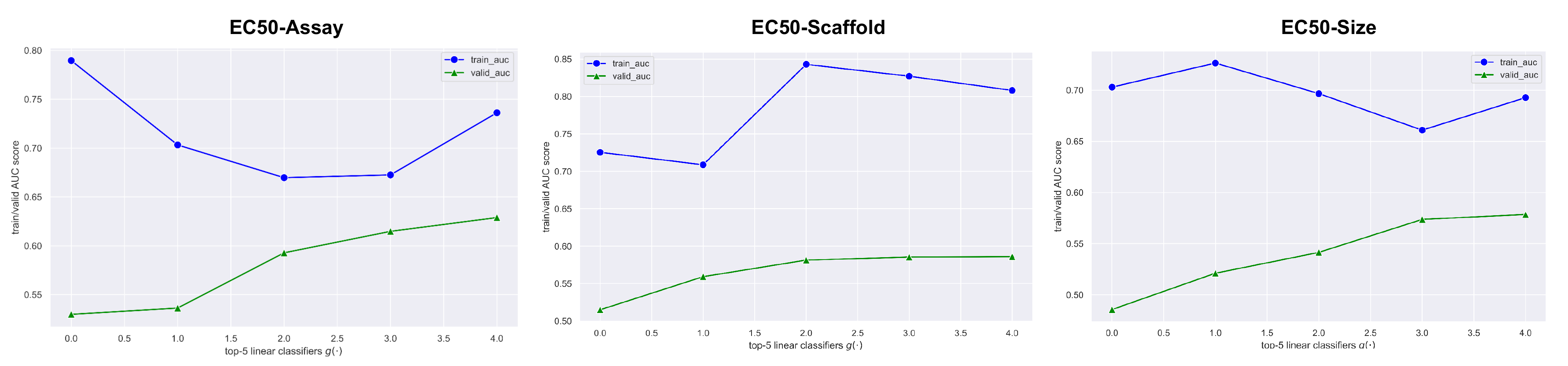} 
    \caption{Train and validation AUC score for top-5 linear classifiers $g(\cdot)$ in step 2 for EC50 datasets.}
    \label{fig:ec50_svm}
\end{figure}

\begin{figure}[t]
\centering
\vspace{-10pt}
\scalebox{0.9}{
\includegraphics[width=0.7\textwidth]{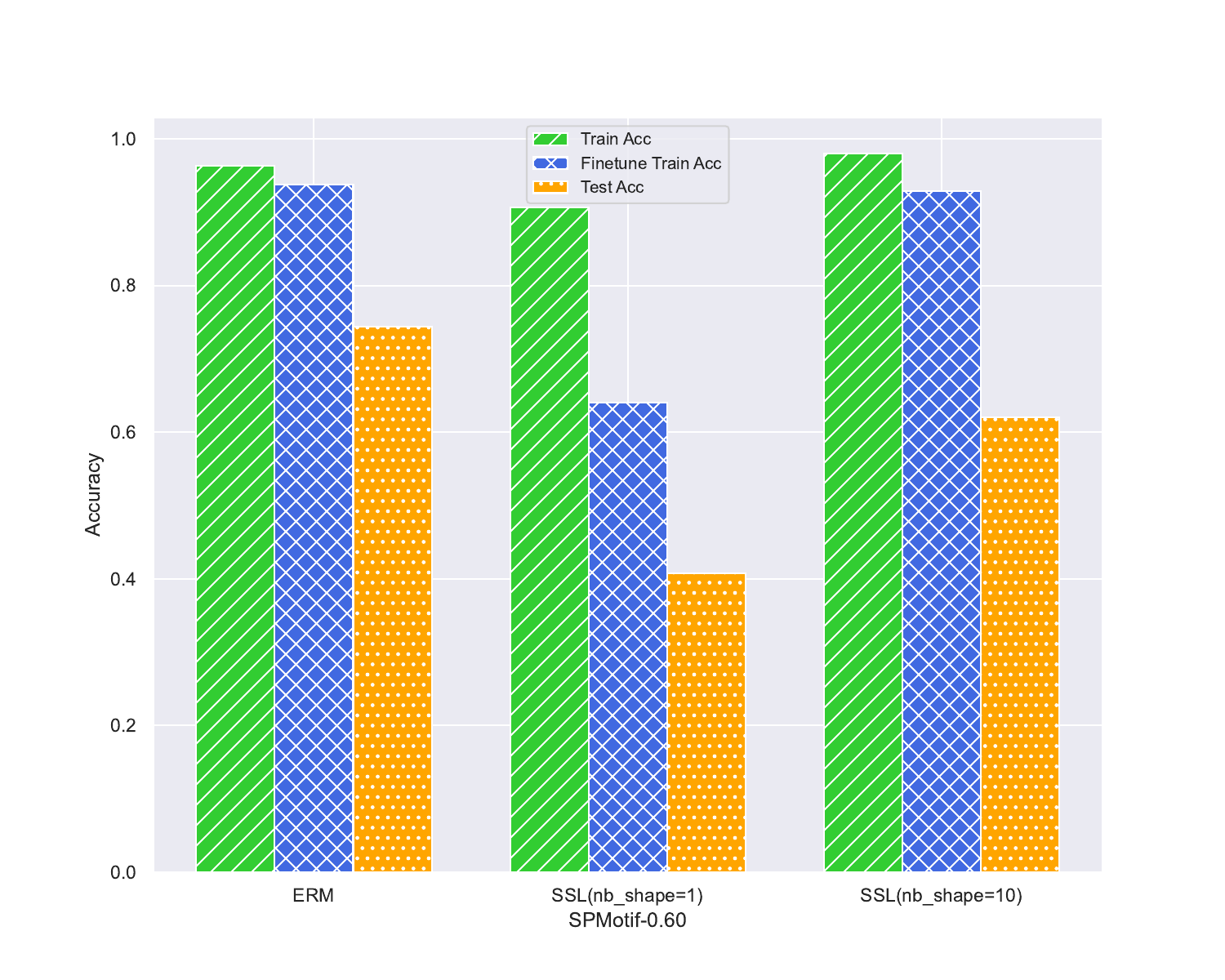}}
\vspace{-10pt}
\caption{Representation quality in terms of how much invariant features are included. When $|G_c|>|G_s|$, infomax-based SSL effectively include causal invariant features, which validates assumptions of theorem~\ref{thm:spu_infomax} empirically from another perspective.} 
\label{fig:gc_bigger}
\end{figure}

\subsection{More Experiment Results}
\label{app:exp}

\paragraph{More Ablation Studies} We provide more ablation studies to study the impact of different SSL algorithms in the \textit{Encoding} stage, and the impact of different decorrelation loss objectives. In our study, we adopt the infomax principle to learn spurious features, here we provide experiment results with GraphCL~\cite{graphcl} as an alternative for learning spurious features in the \textit{Encoding} stage, as illustrated in Table~\ref{app:abstudy_ssl}. As shown in Table \ref{app:abstudy_ssl}, the representations derived from GraphCL, compared to those obtained using Infomax, exhibit a notable decline in performance on both synthetic and real-world datasets. This phenomenon could likely be attributed to the scenario when $|G_c| < |G_s|$, randomly dropping nodes or edges are more likely to perturb the spurious subgraph $G_s$, thereby enhancing the invariant feature learning for $G_c$. This underscores the importance of accurately identifying spurious features in the \textit{encoding} step. We also conduct experiments to study the decorrelation loss objectives (i.e., Eqn.~\ref{final_loss} and Eqn.~\ref{erm_eq_alter}). The experiment results are illustrated in Table~\ref{app:abstudy_loss}. As discussed in Appendix~\ref{app:discuss_equad_framework}, Eqn.~\ref{final_loss} poses optimization challenges due to gradient conflicts for Two-piece graph datasets, while using Eqn.~\ref{erm_eq_alter} brings significant performance improvement, which also demonstrates the flexibility of EQuAD.

\begin{table}[]
\centering
\caption{Ablation study on different SSL algorithms}
\label{app:abstudy_ssl}
\begin{tabular}{@{}cllllll@{}}
\toprule
\textbf{Method} &
  \multicolumn{1}{c}{\textbf{SPMotif-0.40}} &
  \multicolumn{1}{c}{\textbf{SPMotif-0.60}} &
  \multicolumn{1}{c}{\textbf{SPMotif-0.90}} &
  \multicolumn{1}{c}{\textbf{EC50-Assay}} &
  \multicolumn{1}{c}{\textbf{EC50-Sca}} &
  \multicolumn{1}{c}{\textbf{EC50-Size}} \\ \midrule
ERM             & 62.19±3.26          & 55.24±2.43 & 49.41±3.78 & 67.39±2.90 & 64.98±1.29 & 65.10±0.38 \\
EQuAD w/GraphCL & 72.78±0.81          & 70.66±2.67 & 67.79±1.29 & 75.66±1.27 & 66.24±0.61 & 64.13±0.65 \\
EQuAD w/Infomax & \textbf{73.13±1.56} & \textbf{71.93±1.94} & \textbf{69.47±2.06} & \textbf{79.36±0.73} & \textbf{68.12±0.48} & \textbf{66.37±0.64} \\ \bottomrule
\end{tabular}
\end{table}

\begin{table}[]
\centering
\caption{Ablation study on different decorrelation loss objectives}
\label{app:abstudy_loss}
\begin{tabular}{@{}cllllll@{}}
\toprule
\textbf{Method} &
  \multicolumn{1}{c}{\textbf{(0.80,0.70)}} &
  \multicolumn{1}{c}{\textbf{(0.80,0.90)}} &
  \multicolumn{1}{c}{\textbf{(0.70,0.90)}} &
  \multicolumn{1}{c}{\textbf{EC50-Assay}} &
  \multicolumn{1}{c}{\textbf{EC50-Sca}} &
  \multicolumn{1}{c}{\textbf{EC50-Size}} \\ \midrule
EQuAD w/Eqn.~\ref{erm_eq_alter} &
  82.76±0.71 &
  \textbf{75.81±0.51} &
  \textbf{71.95±1.14} &
  79.14±0.51 &
  67.36±0.43 &
  64.48±1.14 \\
EQuAD w/Eqn.~\ref{final_loss} &
  \textbf{82.03±0.38} &
  54.36±3.03 &
  44.69±1.06 &
  \textbf{79.36±0.73} &
  \textbf{68.12±0.48} &
  \textbf{66.37±0.64} \\ \bottomrule
\end{tabular}
\end{table}

\paragraph{Sensitivity Analysis} We first investigate the impact of various hyperparameters. First, we examine $\lambda$ in Eqn.~\ref{final_loss}, which controls the regularization strength of $\mathcal{L}_{Inv}$. As demonstrated in Figure \ref{fig:sl_abstudy}, across 5 different datasets, the optimal $\lambda$ generally remains below $0.1$. When $\lambda \geq 1$, there is a substantial decline in model performance, as now the empirical risk focus on $\mathcal{L}_{Inv}$ which maximizes the conditional entropy $H\left(\mathbf{h}_s \mid \widehat{\mathbf{h}}_c\right)$, hence degrade to random guessing ultimately. The experiment results are also consistent with our theoretical results in Sec.~\ref{proof:theo_inv_eq}, i.e., $\lambda$ should be less than $1.0$ to guide the model focusing on learning invariant features. 

Next, we examine $\gamma$ and $\tau$, the hyperparameters that control the sample-specific reweighting and model-specific reweighting scheme respectively. The experiments are performed on three EC50 datasets. As illustrated in Figure \ref{fig:ec50_hyper_3d}, under a range of hyperparameter settings, EQuAD consistently matches or outperform the performance of state-of-the-art methods in all three EC50 datasets. Moreover, the careful tuning of $\gamma$ and $\tau$ are shown to be effective empirically. For instance, on the EC50-Scaffold dataset, the best-performing hyperparameters yield a performance increase of $3.67\%$ compared to the least effective settings.

\paragraph{More Visualizations on Latent Representations in 2d Space}  We provide more visualization results of latent representations in 2d space using t-SNE~\cite{JMLR:v9:vandermaaten08a}. Figure~\ref{fig:tsne_spm89} illustrates the latent representation in 2d space for the training set, where samples from different labels (represented by various colored points) trained with both ERM and EQuAD exhibit a clear decision boundary, achieving training accuracy above 90\% in both cases. On the other hand, within the same classes, we use \textit{cycles} and \textit{triangles} to represent data samples with high and low spurious correlations to $Y$ respectively. The visualizations indicate that the representations obtained from ERM exhibit a notable separation between the two groups within the latent space, implying minimal overlap. In contrast, the representations from EQuAD show greater overlap, or they are more closely positioned. This observation implies that the GNN model $f(\cdot)$ trained from scratch learns high-quality invariant features.

\paragraph{More Visualizations on Prediction Logits from $g(\cdot)$} We present the logits distributions for all three classes in SPMotif-0.40 and SPMotif-0.80, as derived from the top-4 classifiers $g(\cdot)$ from Step 2. Figure~\ref{fig:boxplots} illustrates that, under varying degrees of spurious correlation, both high and low correlation samples exhibit discernible differences. Notably, at a higher bias level (e.g., $bias=0.8$), the distinction between these samples is more pronounced, with a certain symmetry around 0.5, providing favorable conditions for subsequent decorrelation. However, at $bias=0.4$ some classes show less distinct separation and lack symmetry. This observation underscores the need to implement sample-specific and model-specific reweighting strategies to determine the optimal prediction logits matrices.

\paragraph{Does infomax-based SSL representations capture spurious substructures in real-world datasets?} Our analysis, as illustrated in Figure \ref{fig:emb_quality}, demonstrates that infomax-based SSL embeddings predominantly capture spurious features in SPMotif datasets. This is a crucial prerequisite for successful decorrelation and invariance learning in subsequent steps. To further investigate the propensity of infomax-based SSL embeddings to learn spurious features in real-world datasets, we conducted experiments on IC50 and EC50 datasets. We evaluated the train/validation AUC scores for the top-5 classifiers $g(\cdot)$ in step 2, specifically selecting those with the lowest validation AUC scores. Ideally, we would expect a high train AUC score with a low validation AUC score.

Observations from Figures~\ref{fig:ic50_svm} and \ref{fig:ec50_svm} reveal that, across datasets with different distribution shifts, there is typically a significant gap between train and validation AUC scores for most $g(\cdot)$, indicating a primary focus on learning spurious features. However, there are exceptions, such as in the IC50-Scaffold dataset, where a smaller gap due to lower train AUC scores suggests that the representation may not predominantly contain spurious features. This finding leads us to consider implementing a model-specific reweighting scheme to address such scenarios. 
Finally, the observed discrepancy between the training and validation AUC scores also supports the assumption that $|G_c|<|G_s|$ is valid in these datasets, which is crucial for the successful capture of spurious features by the infomax-based SSL approach.

\paragraph{Infomax principle learns invariant features when $|G_c|>|G_s|$} To further evaluate our theoretical presumption that infomax-based SSL can extract spurious features when $|G_c|<|G_s|$, we perform another empirical study by modifying the SPMotif dataset to let $|G_c|>|G_s|$, and examine whether infomax-based SSL learns causal features. Specifically, we modify SPMotif dataset to shrink the size of base spurious subgraph (from 40 nodes in average to 12) and increase the number of ground-truth subgraphs (from 1 to 10). By doing this, $|G_c|>|G_s|$ is satidfied, and we conduct a similar experiment as discussed in Section~\ref{app:rep_quality}. Figure~\ref{fig:gc_bigger} demonstrates both scenarios of $|G_c|<|G_s|$ (in the middle) and $|G_c|>|G_s|$ (in the right). When $|G_c|>|G_s|$, the validation accuracy (after feature reweighting) exceeds $90\%$, implying that SSL with infomax principle learns causal invariant features. We can also see that the test accuracy is also much higher compared to the case when $|G_c|<|G_s|$ ($62.1\%$ vs $40.3\%$), although still lower than ERM.


\end{document}